\documentclass[11pt]{article}

\usepackage[T1]{fontenc}
\usepackage[utf8]{inputenc}
\usepackage{amsmath,amssymb,amsthm}
\usepackage{mathtools}
\usepackage{booktabs}
\usepackage{array}
\usepackage{graphicx}
\usepackage{xcolor}
\usepackage{hyperref}
\usepackage{cleveref}
\usepackage{geometry}
\usepackage{setspace}
\usepackage[numbers]{natbib}
\usepackage{caption}
\usepackage{subcaption}
\usepackage[section]{placeins}
\usepackage{enumitem}
\usepackage{mdframed}

\newtheorem{theorem}{Theorem}
\newtheorem{proposition}[theorem]{Proposition}
\newtheorem{corollary}[theorem]{Corollary}
\newtheorem{definition}{Definition}

\definecolor{pgsa}{RGB}{0,100,180}
\definecolor{jepa}{RGB}{200,60,40}
\definecolor{pixel}{RGB}{120,120,120}

\title{\textbf{Identifiability Without Gaussianity:\\
Symbolic World Models and Near-Infinite Temporal Consistency}}

\author{
  \textbf{Seth Dobrin}\textsuperscript{1} \and
  \textbf{Lukasz Chmiel}\textsuperscript{1}
}

\date{June 2026}

\begin{document}

\maketitle

{\small\centering
\textsuperscript{1}ARYA Labs, PBC\\
\texttt{\{seth,lukasz\}@aryalabs.io}\par
\vspace{0.5em}
}

\begin{abstract}
Klindt, LeCun, and Balestriero~\cite{klindt2026} proved that
Joint-Embedding Predictive Architectures (JEPAs) achieve linear
identifiability (the linear recovery of the world's true latent
variables) if and only if the world's latent dynamics follow a
Gaussian, stationary process.
This Gaussian boundary implies a fundamental limit on temporal
consistency: for any non-Gaussian physical system, under a
bias-coherence condition made explicit in the paper, the representation
error of a statistical World Model grows monotonically with time.
We prove that this limit is an artifact of the statistical alignment
mechanism, not a property of World Models in general.
We introduce the \textbf{Physics-Grounded Symbolic Architecture}
(PGSA) and prove three results:
(1)~a PGSA achieves exact linear identifiability for all physical
regimes, regardless of the latent distribution;
(2)~the per-step error of a PGSA is bounded by numerical precision
alone; and
(3)~as a direct consequence, a PGSA maintains temporal consistency
for an unbounded number of transitions, a property we term
\textbf{near-infinite temporal consistency}.
We further prove that statistical World Models cannot achieve this
property for any non-Gaussian system, regardless of model capacity or
training data volume: their error is floored, unconditionally, by the
representation bias, measured at more than fourteen orders of
magnitude above machine precision on every system studied.
The algebraic cores of Theorems~\ref{thm:ident}, \ref{thm:temp}, \ref{thm:ceiling}, and~\ref{thm:approx} are formalized in
Lean~4 with Mathlib4 v4.31.0 (zero \texttt{sorry} placeholders);
the Klindt et al.\ converse is taken as an external premise.
The contrast establishes that symbolic grounding in the causal
generator of the world's dynamics is the sufficient condition, and for
non-Gaussian regimes, the only condition for near-infinite
temporal consistency.
\end{abstract}

\section{Introduction}

A World Model is only as useful as it is consistent.
A model that predicts the next state accurately but accumulates error
over time is not modeling the world; it is modeling a local
neighbourhood of the world, beyond which its predictions become
unreliable.
For perception this is a tolerable nuisance.
For planning, control, and scientific simulation, which are the core
uses of a World Model in robotics, autonomous vehicles, and
engineering, it is disqualifying.
Temporal consistency is therefore not a desirable feature of a World
Model. It is a necessary one.

The current effort to build foundational World Models follows two
broad strategies, and a precise temporal-consistency limit applies to
each.

\paragraph{The generative pixel-space strategy.}
Models in this family, including high-capacity video generators such
as Sora, Veo, and Cosmos, treat a scene as a sequence of pixels or
volumetric tokens and predict the next frame autoregressively.
The strategy is expressive, but perception and dynamics are entangled
in a single statistical mechanism trained to match the data
distribution.
Every single-step prediction carries a representation error that the
next step consumes as input, so error accumulates along the rollout,
at least linearly for energy-conserving dynamics and exponentially
once the system is chaotic.
We make this mechanism precise in
Proposition~\ref{prop:blind} and Corollary~\ref{cor:pixel}: a model that cannot separate the act of
seeing from the law of motion inherits a finite consistency horizon
for a reason more basic than capacity.

\paragraph{The latent predictive strategy.}
Joint-Embedding Predictive Architectures (JEPA, I-JEPA, V-JEPA) and
related energy-based predictors avoid pixel noise by predicting in an
abstract latent space rather than in observation space.
This removes the perceptual-reconstruction burden, but it does not
remove the consistency limit; it relocates it.
Klindt, LeCun, and Balestriero~\cite{klindt2026} proved a landmark
result that states the limit exactly: LeJEPA, a JEPA with Gaussian
regularization, achieves \emph{linear identifiability}, the linear
recovery of the world's true latent variables, if and only if the
world's latents are Gaussian and evolve under a stationary,
additive-noise (Ornstein-Uhlenbeck) transition~\cite{uhlenbeck1930}.
Their Theorem~5.3 shows that the recovery degrades gracefully when the
training objectives are met only approximately.

\paragraph{Why the Gaussian boundary bites.}
The Ornstein-Uhlenbeck transition
$z' = \rho z + \sqrt{1-\rho^2}\,\eta$, with
$\eta \sim \mathcal{N}(0,I_n)$, is mean-reverting by construction, and
it is the only stationary additive-noise transition for a Gaussian
latent.
The physical world is not mean-reverting.
A projectile does not return to its launch point, a phase transition
does not undo itself, and turbulence does not relax to its initial
condition.
For any system whose latent dynamics are not Ornstein-Uhlenbeck, which
is to say for the great majority of physical systems of scientific and
engineering interest, the identifiability guarantee fails, the learned
representation is a nonlinear distortion of the true latents, and,
when the induced per-step biases do not cancel, the representation
error grows without bound as the rollout lengthens.
The limit is not a defect of any one network. It is a property of
statistical alignment on a non-Gaussian world.

\paragraph{This paper.}
We show that near-infinite temporal consistency is achievable, and
that it does not come from a better statistical objective.
It comes from grounding the architecture in the causal generator of
the world's dynamics.
We introduce the \textbf{Physics-Grounded Symbolic Architecture}
(PGSA), which represents state in a symbolic basis of physical
quantities and advances it with the system's own governing equations
rather than with a learned transition fitted to data.
Our contributions are:

\begin{enumerate}[itemsep=2pt,topsep=3pt]
\item \textbf{Identifiability without Gaussianity.}
A PGSA achieves exact linear identifiability for any latent
distribution, with no Gaussianity or stationarity assumption,
conditional only on the existence of a causal basis for the dynamics
(Definition~\ref{def:basis} and Theorem~\ref{thm:ident}).
\item \textbf{A machine-precision error floor.}
The per-step error of a PGSA is bounded by numerical precision $\mu$
alone, not by a representation bias (Theorem~\ref{thm:temp}).
\item \textbf{Near-infinite temporal consistency.}
For any non-chaotic physical system the consistency horizon is set by
$\mu$ and the system's Lipschitz constant rather than by the number of
steps, so it exceeds any practical rollout by many orders of magnitude
(Theorems~\ref{thm:ceiling} and~\ref{thm:order}).
For chaotic systems the horizon stays finite for every architecture;
symbolic grounding extends it by replacing a representation bias
$\kappa(p)$ with machine precision $\mu$ inside the Lyapunov time.
\item \textbf{A matching bound for statistical models.}
No statistical World Model trained by spectral alignment can reach
this regime on a non-Gaussian world, at any model capacity or data
volume (Theorem~\ref{thm:ceiling}).
Attention does not change the conclusion: it enlarges the constant in
the horizon but cannot remove the ceiling (Corollary~\ref{cor:attention}).
\item \textbf{Monotone improvement without retraining.}
As the symbolic basis (the Atom Registry) grows, the error bound
$L_{\mathrm{known}} \cdot M$ tightens as coverage reduces the uncovered
residual $M$, without any retraining (Theorem~\ref{thm:approx}).
\end{enumerate}

The algebraic cores of
Theorems~\ref{thm:ident}, \ref{thm:temp}, \ref{thm:ceiling}, and~\ref{thm:approx} are machine-checked
in Lean~4 with Mathlib4 v4.31.0 and carry zero \texttt{sorry}
placeholders; the Klindt et al.\ converse enters as a cited external
premise rather than a re-proved result.
The theory is instantiated on seven dynamical systems integrated to
numerical tolerance and on four engineering systems solved by finite
elements, all reproducible from the accompanying artifact.
Section~\ref{sec:pgsa} defines the architecture; the remaining sections
develop the formal results, the measurements, and the limitations.

\section{The Temporal Consistency Problem and Related Work}\label{sec:problem}

\paragraph{Setup.}
Let the world's true latent variables be $z \in \mathbb{R}^n$ and let
$W\colon \mathbb{R}^n \to \mathbb{R}^n$ be the world's true transition.
A statistical World Model learns a representation $h = f \circ g$ and a
predicted transition $\hat{W}$, and rolls out by composing $\hat{W}$
in the learned space.
The single-step error is
$\epsilon_1(z) = \|W(z) - \hat{W}(h(z))\|_2$, and $\epsilon_t$ denotes
the accumulated error after $t$ steps.
We measure consistency by the horizon
$T^*(\delta) = \min\{t : \epsilon_t > \delta\}$, the first step at
which the error exceeds a fixed tolerance $\delta$.

\paragraph{The latent bound.}
For LeJEPA in a Gaussian world, Klindt et al.~\cite{klindt2026} prove
$\epsilon_1 = 0$ at the optimum (their Theorem~5.1) and that the
Gaussian is the unique latent distribution with this property (their
Theorem~5.2).
For a non-Gaussian world the optimal representation is therefore a
nonlinear distortion of the true latents, and the residual bias is
captured by a single scalar,
$\kappa(p) = \sup_z \|h^*(z) - z\|_2 > 0$.
Because the learned transition reapplies this bias at every step, the
statistical horizon is $T^*_{\mathrm{stat}} = \delta/\kappa(p)$:
finite whenever $\kappa(p) > 0$, and independent of model size or
training-set size.
Theorem~\ref{thm:ceiling} states this bound, and
Figure~\ref{fig:lejepa_comparison} shows the measured latent
distributions that violate the Gaussian premise; the measured identifiability of the seven
physical systems, computed from the closed-form population optimum
with no encoder trained, is reported in Table~\ref{tab:linident}.

\paragraph{The pixel-space mechanism.}
Generative pixel-space models face the latent bound and an additional
one.
Perception and dynamics are not separated: one statistical mechanism
maps observations to a representation and predicts the next
observation.
Proposition~\ref{prop:blind} and Corollary~\ref{cor:pixel} show that such a model cannot in general
recover the causal state, so its single-step error is bounded below by
the perceptual mismatch and its rollout error compounds, linearly for
energy-conserving dynamics and exponentially for chaotic dynamics.
This is why pixel rollouts drift on a timescale far shorter than the
latent bound alone would predict.

\paragraph{The role of attention.}
Transformer World Models with statistically aligned latents inherit
the same ceiling.
Attention couples states across a context window and can lower the
effective per-step bias, but it cannot drive that bias to zero on a
non-Gaussian world.
Corollary~\ref{cor:attention} makes this quantitative: attention multiplies the
horizon by a constant set by the context length and leaves the
finiteness of $T^*_{\mathrm{stat}}$ unchanged.

\paragraph{Physics-informed learning.}
A substantial line of work injects physical structure into learned
models.
Physics-informed neural networks enforce governing equations as soft
penalties~\cite{raissi2019}.
Hamiltonian and Lagrangian networks build conservation laws into the
architecture~\cite{greydanus2019,lutter2019}.
Koopman and operator methods learn linear representations of nonlinear
dynamics~\cite{lusch2018}, and symbolic regression recovers governing
equations from data~\cite{cranmer2020}.
These methods reduce error but do not establish identifiability
without a distributional assumption; the identifiability question is
the one Klindt et al.\ answered for the Gaussian case and that
nonlinear independent component analysis studies more
generally~\cite{hyvarinen2016,khemakhem2020}.
The present work removes the Gaussianity requirement entirely by
grounding the representation in a causal basis of the dynamics, and
bounds the resulting horizon.

\begin{figure*}[!ht]
\centering
\includegraphics[width=\textwidth,height=0.82\textheight,keepaspectratio]{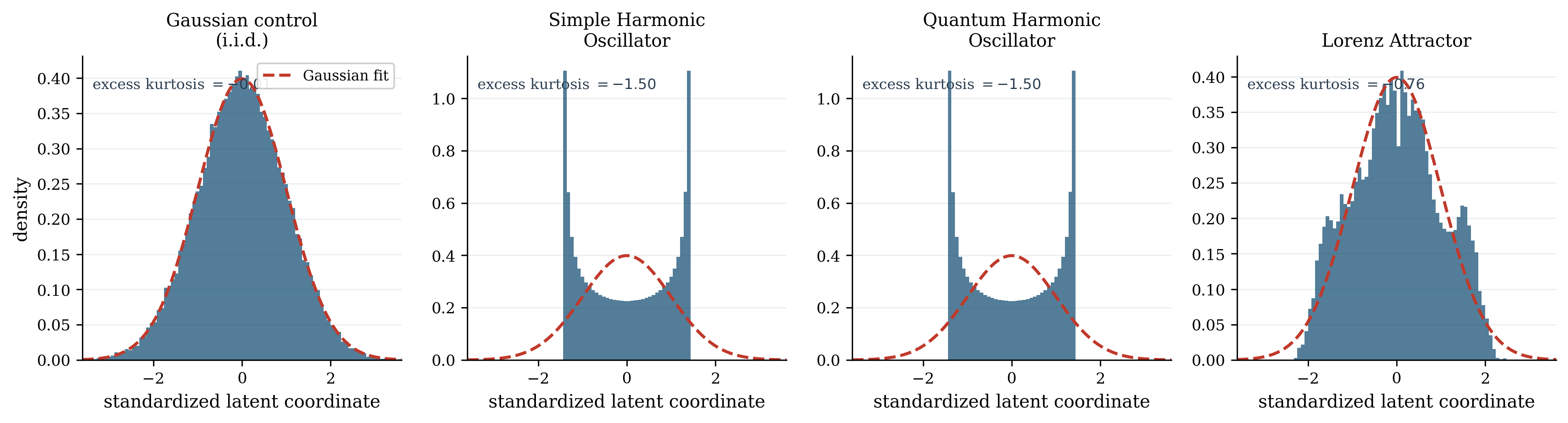}
\caption{\textbf{Measured latent distributions violate the Gaussian premise.}
Standardized marginal distributions of a latent coordinate for an
i.i.d.\ Gaussian control and three of the seven measured physical systems,
each overlaid with its best-fit Gaussian (dashed). Excess kurtosis is
computed from the plotted trajectories; zero is the Gaussian value.
The identifiability guarantee of Klindt et al.~\cite{klindt2026} applies
to the control. Every measured physical system departs from it, and the
measured consequence for linear identifiability is reported in
Figure~\ref{fig:recovery_pgsa} and Table~\ref{tab:linident}.}
\label{fig:lejepa_comparison}
\end{figure*}

\begin{proposition}[Statistical Temporal Divergence]\label{prop:diverge}
For any statistical World Model trained via spectral alignment on a
non-Gaussian world, under the bias-coherence condition of
Assumption~1 (Section~\ref{sec:bridge}), there exists a finite time $T^*$ such that
$\epsilon_t > \delta$ for all $t > T^*$ and any fixed tolerance
$\delta > 0$.
\end{proposition}

The proof follows directly from Klindt et al.\ Theorem~5.2 and is
given in Appendix~\ref{meth:proofs}.

\section{The Physics-Grounded Symbolic Architecture}\label{sec:pgsa}

We define a \textbf{Physics-Grounded Symbolic Architecture} (PGSA)
by three components.

\begin{definition}[Atom Registry]
Let $\mathcal{A} = \{a_1,\dots,a_m\}$ be a finite set of
deterministic, executable functions, where each
$a_i\colon \mathcal{D}_i \subseteq \mathbb{R}^{k_i} \to \mathbb{R}$
represents a known physical law with a well-defined domain
$\mathcal{D}_i$.
Each atom is time-invariant.
\end{definition}

\begin{definition}[State Graph]
Let $G = (V,E)$ be a directed graph where each node $v \in V$ is a
typed physical variable (mass, velocity, force, \ldots) and each
directed edge $(u,v) \in E$ is labelled by an atom $a_i \in
\mathcal{A}$ such that $a_i$ can compute $v$ from $u$.
\end{definition}

\begin{definition}[Causal Basis]\label{def:basis}
The Atom Registry $\mathcal{A}$ contains a \emph{causal basis} for a
world generator $W$ if there exists a composition
$C = a_k \circ \cdots \circ a_1 \in \mathcal{A}^*$ such that
$C(z) = W(z)$ for all $z \in \mathcal{D}$, where $\mathcal{D}$ is a
dense subset of $\mathrm{dom}(W)$.
\end{definition}

\section{Main Results}

\begin{theorem}[Symbolic Identifiability]\label{thm:ident}
Let $\mathcal{A}$ contain a causal basis for $W$.
Then the PGSA achieves exact identifiability:
\[
  h(z) = z \quad \text{and} \quad \epsilon_1(z) = 0
  \quad \text{for all } z \in \mathcal{D},
\]
regardless of the probability distribution $p(z)$.
\end{theorem}

Figure~\ref{fig:recovery_pgsa} shows what Theorem~\ref{thm:ident} means
on the seven measured systems.
The Gaussianizing representation that a statistical model forms corrupts
the geometry of each bounded orbit, turning a closed orbit into a pinched
astroid, while symbolic inversion on the causal basis recovers the state
exactly.
The static distortion is modest in $R^2$, about $0.90$ on the smooth
systems, but it is a fixed bias, and Theorem~\ref{thm:ceiling} and
Figure~\ref{fig:error_growth} show how that bias compounds over a rollout.
A second pattern in the same figure is worth stating plainly:
identifiability and temporal consistency run in opposite directions across
the suite.
The two chaotic systems are near-Gaussian, so they are the easiest to
identify ($R^2 \approx 0.99$), yet their temporal horizon is the shortest,
finite at about 31 Lyapunov times (Theorem~\ref{thm:ceiling}).
The smooth systems are more non-Gaussian, so they are harder to identify
($R^2 \approx 0.90$), yet their horizon is near-infinite.

\begin{figure*}[!ht]
\centering
\includegraphics[width=\textwidth,keepaspectratio]{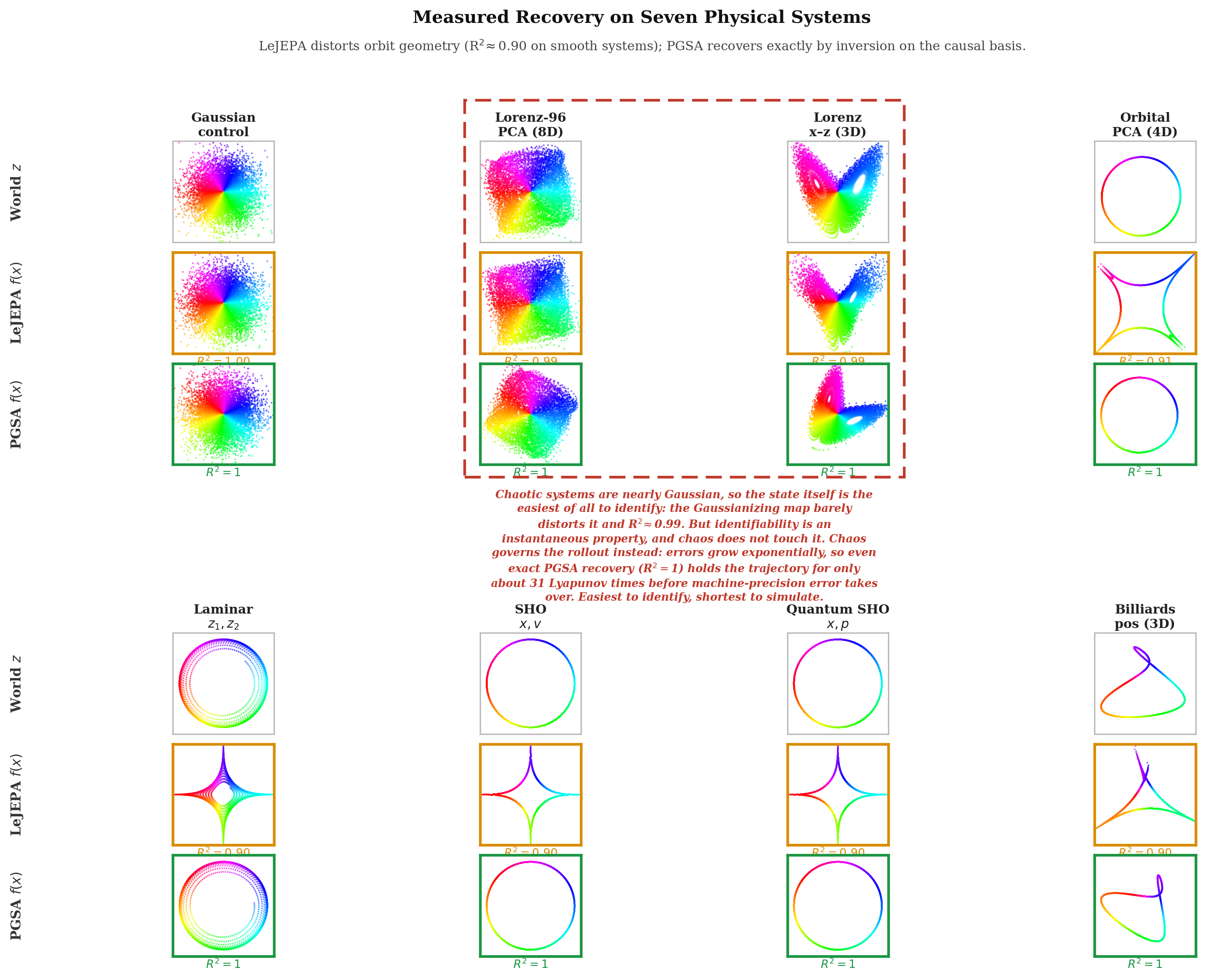}
\caption{\textbf{Measured recovery on seven physical systems.}
For each system, the true latent state (top row of each triptych), the
LeJEPA representation (middle), and the PGSA recovery (bottom), each shown
as a labeled 2D projection of the standardized state.
The LeJEPA representation is the closed-form population-optimal
isotropic-Gaussian map, the same construction behind
Table~\ref{tab:linident}; it Gaussianizes each marginal and corrupts the
orbit geometry, turning a closed orbit into a pinched astroid, while
preserving most linear variance (linear identifiability $R^2 \approx 0.90$
on the smooth systems, $0.99$ on the chaotic ones).
PGSA inverts the observation map on the causal basis
(Definition~\ref{def:basis}) and recovers the state exactly ($R^2 = 1$, by
construction; Theorem~\ref{thm:ident}).
The two chaotic systems are boxed: they are near-Gaussian and the easiest
to identify, yet their temporal horizon is finite at about 31 Lyapunov
times (Theorem~\ref{thm:ceiling}), whereas the smooth systems are harder to
identify but have a near-infinite horizon.
Identifiability and temporal consistency therefore run in opposite
directions across the suite.
$R^2$ values reproduce Table~\ref{tab:linident}.}
\label{fig:recovery_pgsa}
\end{figure*}

\begin{theorem}[Near-Infinite Temporal Consistency]\label{thm:temp}
Let $\mathcal{A}$ contain a causal basis for $W$.
Let $\mu$ be the machine precision of the numerical implementation
of $\mathcal{A}$, and let $z^{(0)} \in \mathrm{dom}(W^t)$.
Then for all $t \geq 1$, the accumulated rounding error between the
true trajectory $W^t(z^{(0)})$ and the numerically implemented
trajectory $\hat{W}^t(z^{(0)})$ satisfies:
\[
  \epsilon_t(z^{(0)}) \;\leq\; \mu \sum_{k=0}^{t-1} \|W\|_{\mathrm{Lip}}^{k}
  \;=\;
  \begin{cases}
    t \cdot \mu, & \|W\|_{\mathrm{Lip}} = 1,\\[4pt]
    \mu \, \dfrac{\|W\|_{\mathrm{Lip}}^{t} - 1}{\|W\|_{\mathrm{Lip}} - 1}, & \text{otherwise}.
  \end{cases}
\]
\end{theorem}

\subsection{From Identifiability to Temporal Consistency}
\label{sec:bridge}

The central logical bridge of this paper connects a static representational property (non-identifiability) to a dynamic execution failure (temporal divergence). We formalize this connection before stating Theorem~\ref{thm:ceiling}.

\begin{mdframed}[backgroundcolor=gray!10, linewidth=0pt, innertopmargin=8pt, innerbottommargin=8pt]
\textbf{Lemma 1 (Representation Bias Propagation).} Let $h^*(z) = z + b(z)$ be an optimal but non-identifiable representation with bias $b(z) \neq 0$. Let $\hat{W}$ be a transition function trained on the representation space. Then, to first order in the bias, the temporal error after $t$ steps satisfies:
\[
  \epsilon_t \;\geq\; \left\| \sum_{i=0}^{t-1} \hat{W}^i\!\bigl(b(z^{(t-1-i)})\bigr) \right\|_2.
\]
\end{mdframed}

\begin{mdframed}[backgroundcolor=gray!10, linewidth=0pt, innertopmargin=8pt, innerbottommargin=8pt]
\textbf{Assumption 1 (Bias Coherence).} Along trajectories in the support of $p$, the propagated per-step biases do not systematically cancel: there exists an effective per-step bias $\kappa(p) \in \bigl(0, \sup_z \|b(z)\|_2\bigr]$ such that
\[
  \Bigl\| \sum_{i=0}^{t-1} \hat{W}^i\!\bigl(b(z^{(t-1-i)})\bigr) \Bigr\|_2 \;\geq\; t \cdot \kappa(p)
  \quad \text{for all } t \geq 1.
\]
\end{mdframed}

This lemma is the formal core of the paper, so it is worth saying in words what it claims and why it holds. A non-identifiable representation does not merely mislabel the latent space at a single instant. It installs a fixed, state-dependent bias $b(z)$ that the model cannot remove, because that bias is what the training objective settled on as optimal. When the model rolls forward, the learned transition operator $\hat{W}$ has no access to the true state $z$; it sees only the biased image $z + b(z)$, and it must produce the next state from that distorted input. The output is therefore distorted by the operator's response to $b(z)$, and that distorted output becomes the input to the next step. The bias re-enters the computation at every step rather than being paid once, and the summation in Lemma~1 is the accumulated record of those re-entries.

The consequence, under Assumption~1, is qualitative, not a matter of degree. A model with a small per-step bias does not have a small long-horizon error; it has a long-horizon error that grows until it passes any tolerance $\delta$ one cares to set. Coherence is stated as an assumption rather than proven because the summands in Lemma~1 are vectors and could in principle cancel; cancellation, however, would require the learned operator's response to the bias to reverse sign along the trajectory in a way the alignment objective never optimizes for. Figure~\ref{fig:error_growth} illustrates the two error regimes the assumption separates. Capacity does not change this, because $b(z)$ is a property of the optimum the objective defines, not an approximation error that more parameters or more data would shrink. This is the precise sense in which the failure is structural: it is built into what the representation is, before any question of how well the transition operator is trained.

\begin{theorem}[Statistical Temporal Ceiling]\label{thm:ceiling}
Let $M$ be any statistical World Model trained via spectral
alignment, let $p(z)$ be non-Gaussian, and suppose Assumption~1 holds.
Then there exists a finite time $T^*(M,p)$ such that
$\epsilon_t^M(z^{(0)}) > \delta$ for all $t > T^*$ and any fixed
$\delta > 0$.
For conservative systems ($\|W\|_{\mathrm{Lip}} = 1$):
\[
  T^*_{\mathrm{stat}} = \frac{\delta}{\kappa(p)},
\]
where $\kappa(p) > 0$ is the effective per-step representation bias of
Assumption~1, bounded above by the representation bias
$\sup_z \|h^*(z) - z\|_2$ induced by the Gaussian prior, which is
strictly positive for non-Gaussian $p$ by the Klindt et al.\ converse.
No increase in model capacity can eliminate $T^*$.
\end{theorem}

Theorem~\ref{thm:ceiling} names the quantity that sets the limit. The effective bias $\kappa(p)$ is the per-step rate at which the gap between the true latent and the optimal representation the Gaussian prior permits accumulates along a rollout. It is strictly positive for every non-Gaussian $p$, by the Klindt et al.\ converse together with Assumption~1, and it is a property of the objective rather than of any particular trained model. The horizon $T^*_{\mathrm{stat}} = \delta/\kappa(p)$ that follows is finite for every such system, and the clause that no increase in capacity can eliminate $T^*$ is the formal statement that this is a ceiling and not a bottleneck: a larger model approaches the same biased optimum faster, it does not approach a different one.

Assumption~1 carries only the growth claim. The impossibility of
near-infinite temporal consistency for statistical models needs no
assumption at all, because the error can never fall below the bias the
optimum installs.

\begin{proposition}[Representation Bias Floor]\label{prop:floor}
Let $h^*$ be the optimal representation under spectral alignment for a
non-Gaussian $p(z)$, and let the model's state estimate at time $t$ be
any linear readout $\hat{z}_t = A\,h^*(z_t) + c$ of the current
representation, the most favorable case, in which the model re-encodes
the world at every step. Then, unconditionally, for all $t \geq 0$:
\[
  \mathbb{E}_p\bigl[\epsilon_t^2\bigr] \;\geq\;
  \sigma^2_{\mathrm{res}}(p) \;:=\;
  \min_{A,\,c}\; \mathbb{E}_p\bigl\|A\,h^*(z) + c - z\bigr\|_2^2
  \;>\; 0.
\]
The floor is a property of the objective's optimum: it is independent
of the transition model, model capacity, and training data volume, and
it is strictly positive whenever $h^*$ is not a linear function of
$z$, which is the Klindt et al.\ converse for non-Gaussian $p$. In
standardized units the per-coordinate floor is
$\sigma_{\mathrm{res}} = \sqrt{1 - R^2}$, so the measured
identifiability of Table~\ref{tab:linident} instantiates it directly:
across the seven physical systems $\sigma_{\mathrm{res}}$ lies
between $0.077$ and $0.32$, more than fourteen orders of magnitude
above machine precision $\mu \approx 2.2 \times 10^{-16}$.
\end{proposition}

For tolerances $\delta$ below the floor, the statistical horizon is
zero with no coherence assumption; Assumption~1 extends the breach to
every tolerance.

The contrast with a physics-grounded architecture is the reason the two paradigms separate so sharply, and it is a difference in where error enters, not in how carefully either model is built. A PGSA also incurs an error when it first reads the world, the one-time perceptual error $e_0$ of mapping an observation onto typed physical variables. But once the state is parsed, the forward operator is the physical law itself, which acts on the true state rather than on a biased image of it. The perceptual error is therefore paid once at $t = 0$ and is then merely transported by the dynamics, bounded by $L^t\|e_0\|$, which for a conservative system ($L = 1$) stays within machine precision for the entire rollout. A statistical model pays its representational error not once but at every step, because the bias is reapplied each time the operator reads the latent state. One architecture amortizes a single bounded error across the trajectory; the other compounds an irreducible bias along it. Theorem~\ref{thm:order} states this separation precisely.

\begin{theorem}[Temporal Consistency Ordering]\label{thm:order}
For any physically realistic, non-Gaussian system with unobservable
latent dynamics ($\lambda_\perp > 0$) and $\kappa(p) \gg \mu$:
\[
  T^*_\pi \;\leq\; T^*_{\mathrm{stat}} \;<\; T^*_{\mathrm{PGSA}}.
\]
The ratio for conservative systems is universal:
\[
  \frac{T^*_{\mathrm{PGSA}}}{T^*_{\mathrm{stat}}}
  = \frac{\kappa(p)}{\mu}.
\]
For a system where a statistical model achieves an effective $\kappa = 0.01$, this ratio is
$\approx 4.5 \times 10^{13}$.
\end{theorem}

The ordering $T^*_\pi \leq T^*_{\mathrm{stat}} < T^*_{\mathrm{PGSA}}$ places the three paradigms on a single axis. Pixel-space models sit at the bottom because they lose causal variables before any transition is applied; latent-space statistical models sit above them because they retain the variables but encode them with the bias $\kappa(p)$; and a PGSA sits far above both because its only error source is numerical. The ratio $\kappa(p)/\mu$ between the statistical and symbolic horizons is universal for conservative systems, and for a strongly aligned statistical model with an effective $\kappa = 0.01$ it is about $4.5 \times 10^{13}$. Near-infinite is meant in this concrete sense: the symbolic horizon is not literally unbounded, it is bounded by machine precision $\mu \approx 2.2 \times 10^{-16}$ rather than by representation bias, which moves the breach of a fixed tolerance from a few hundred or a few thousand steps to of order $10^{13}$. The measured horizons reported in Appendix~\ref{app:measure} (Table~\ref{tab:steps}) are the empirical form of this ordering across seven physical systems.

\begin{figure}[!htbp]
\centering
\includegraphics[width=\columnwidth]{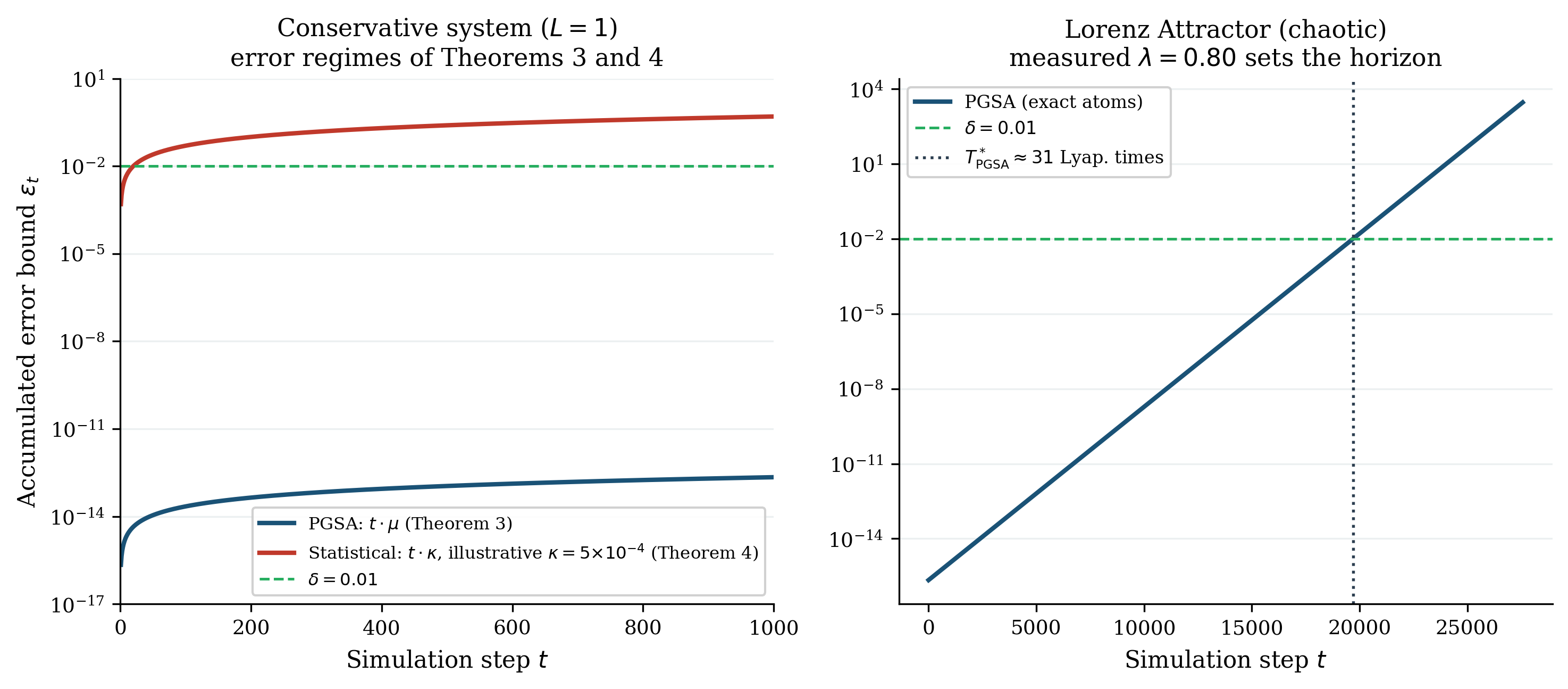}
\caption{The two error regimes of Theorems~\ref{thm:temp}
and~\ref{thm:ceiling}. Left: conservative system; the PGSA accumulated
rounding bound $t \cdot \mu$ against the statistical bias accumulation
$t \cdot \kappa$ under Assumption~1, drawn with an illustrative
$\kappa = 5 \times 10^{-4}$. Right: chaotic horizon with the exponential
envelope parameterized by the measured Lorenz Lyapunov exponent.
Simulation steps on the $x$-axis; error bound $\epsilon_t$ (log scale)
on the $y$-axis.}
\label{fig:error_growth}
\end{figure}

Proofs of all results are given in Appendix~\ref{meth:proofs}.
The algebraic cores of Theorems~\ref{thm:ident}, \ref{thm:temp}, \ref{thm:ceiling}, and~\ref{thm:approx}
are formalized in Lean~4 with Mathlib4~\cite{mathlib2024}.

\section{The Three World Model Paradigms}
\label{sec:paradigms}

World modeling today is pursued through three architectural families: pixel-space models that predict directly in observation space, latent-space models that align a learned representation and predict within it, and transformer world models that attend over a learned token sequence. The results of Section~\ref{sec:bridge} bound each of them, and the three are not independent. We take them in order of increasing architectural sophistication and show that the transformer case is not separate from the latent-space case but a specific instance of it.

\subsection{Pixel-Space Models}

Modern world models rarely operate as pure pixel predictors. Memory-augmented architectures, transformer world models, and video foundation models all employ richer observation encoders. However, the fundamental limitation of observational compression remains: if causal variables are hidden from the observation space, no architecture can recover them without a priori physical knowledge. Pixel-space models fail for a more fundamental reason than latent-space models.

Two formal results make this precise. When the rendering function $g$ maps distinct latent states to the same observation, a pixel-space model cannot distinguish them and so cannot predict the future trajectory of either (Proposition~\ref{prop:blind}). It follows that no pixel-space model can linearly identify the true latents when $g$ is not injective, and the failure is structural rather than empirical (Corollary~\ref{cor:pixel}). This non-injectivity is generic, not a pathological edge case: by the rank-nullity theorem any linear rendering from $\mathbb{R}^n$ to $\mathbb{R}^d$ with $n > d$ has a null space of dimension at least $n - d > 0$, and mass, charge, internal temperature, and quantum state are physical latents invisible in pixel space. The formal statements and the rank-nullity argument are given in Appendix~\ref{meth:proofs}.

\subsection{Latent-Space Models}

Latent-space models are the general statistical case analyzed in Section~\ref{sec:bridge}. Instead of predicting in observation space, they learn a representation by statistical alignment and roll the dynamics forward within it. Theorem~\ref{thm:ceiling} governs them directly: for any non-Gaussian system the optimal aligned representation carries a bias $\kappa(p) > 0$, the per-step error compounds as the transition operator reapplies that bias, and the consistency horizon is $T^*_{\mathrm{stat}} = \delta/\kappa(p)$, finite for every such system. This is the family that includes JEPA and its descendants; the LeJEPA horizon $T^*_{\mathrm{JEPA}}$ is the instance of $T^*_{\mathrm{stat}}$ for that model. It sits above pixel-space models in the ordering of Theorem~\ref{thm:order}, because it retains the causal variables even though it encodes them with bias.

\subsection{Transformer World Models}
\label{sec:transformers}

The field of world modeling has increasingly adopted transformer-based foundation models, including video generation models (Sora, Genie, Cosmos) and memory-augmented simulators (Dreamer~\cite{hafner2025}), and these architectures sit squarely within the latent-space paradigm. They rely on a learned latent in two ways. Directly, a JEPA-style embedding objective produces the tokens the transformer attends over. Indirectly, a variational autoencoder bottleneck compresses observations into the latent the transformer then models. In either case the representation the dynamics operate on is a statistically aligned encoder, so the bias $\kappa(p)$ enters before a single attention head runs.

Because that representation is a statistically aligned encoder in exactly the sense of Theorem~\ref{thm:ceiling}, the bias $\kappa(p)$ is present before attention is applied, and the attention stack operates on a distorted causal basis. No amount of attention recovers information the encoder did not preserve. This paper does not claim a limitation for all transformer architectures; the claim is about any transformer whose latent is formed by statistical alignment, which covers the dominant designs.

\begin{corollary}[Attention Does Not Remove the Statistical Ceiling]\label{cor:attention}
Let $M_T$ be a transformer World Model whose latent encoder is trained
by statistical alignment on a non-Gaussian system, and let attention act
on the encoded tokens. Then $M_T$ is subject to Theorem~\ref{thm:ceiling}:
there exists a finite horizon $T^*(M_T,p)$ beyond which
$\epsilon_t > \delta$, and attention affects only the constant in $T^*$,
not its existence.
\end{corollary}

The corollary isolates what attention can and cannot do. A long context window is a genuine capability: by conditioning on a stretch of history rather than on a single state, a transformer learns higher-order structure that a Markovian one-step model misses, and it uses that structure to predict more accurately over short and medium horizons. This is the same compensation that lets any well-trained statistical model stay consistent for a while before it diverges, and it raises the constant in $T^*$, sometimes substantially. What it cannot do is set $\kappa(p)$ to zero. If two distinct physical states map to the same token distribution, attention over those tokens inherits the ambiguity, and autoregressive generation compounds it exactly as the latent-space argument predicts. The empirical mechanism by which context windows buy this finite horizon is developed in Supplementary Information. A transformer therefore does not occupy a separate point on the ordering of Theorem~\ref{thm:order}; it is a latent-space model with a larger constant, bounded by the same ceiling.

\section{The Formal Comparison}

\subsection{Measured Linear Identifiability}

Before comparing temporal horizons we measure the static quantity they rest on:
how linearly identifiable the true state is from the optimal isotropic-Gaussian
representation, the property Klindt et al.\ tie to Gaussianity. We measure it
without training an encoder, using the closed-form optimal representation, the
per-coordinate Gaussianizing transport, and recording the $R^2$ of the best
linear readout of the state from it (Methods). On a controlled
generalized-normal family the identifiability is exactly $1$ at the Gaussian and
falls off as the latent departs Gaussian in either tail
(Figure~\ref{fig:linident}a), reproducing the Klindt et al.\ identifiability law
on a population-optimal representation rather than a trained one. The seven
systems occupy the non-Gaussian regime (Figure~\ref{fig:linident}b): the
non-chaotic systems carry the largest identifiability gap ($R^2 \approx 0.90$),
and the chaotic systems sit nearer Gaussian ($R^2 \approx 0.99$) and are limited
instead by their Lyapunov time. The measured non-Gaussianity matches
Table~\ref{tab:steps} to three decimal places, so this rests on the same
trajectories as the horizon analysis (Table~\ref{tab:linident}, Methods).
Because the representation is the population optimum, the measured $R^2$ is an
upper bound on what any trained Gaussian-prior encoder can reach: a trained
encoder's identifiability can fall below these values, not above them.

\begin{figure*}[!ht]
\centering
\includegraphics[width=\textwidth]{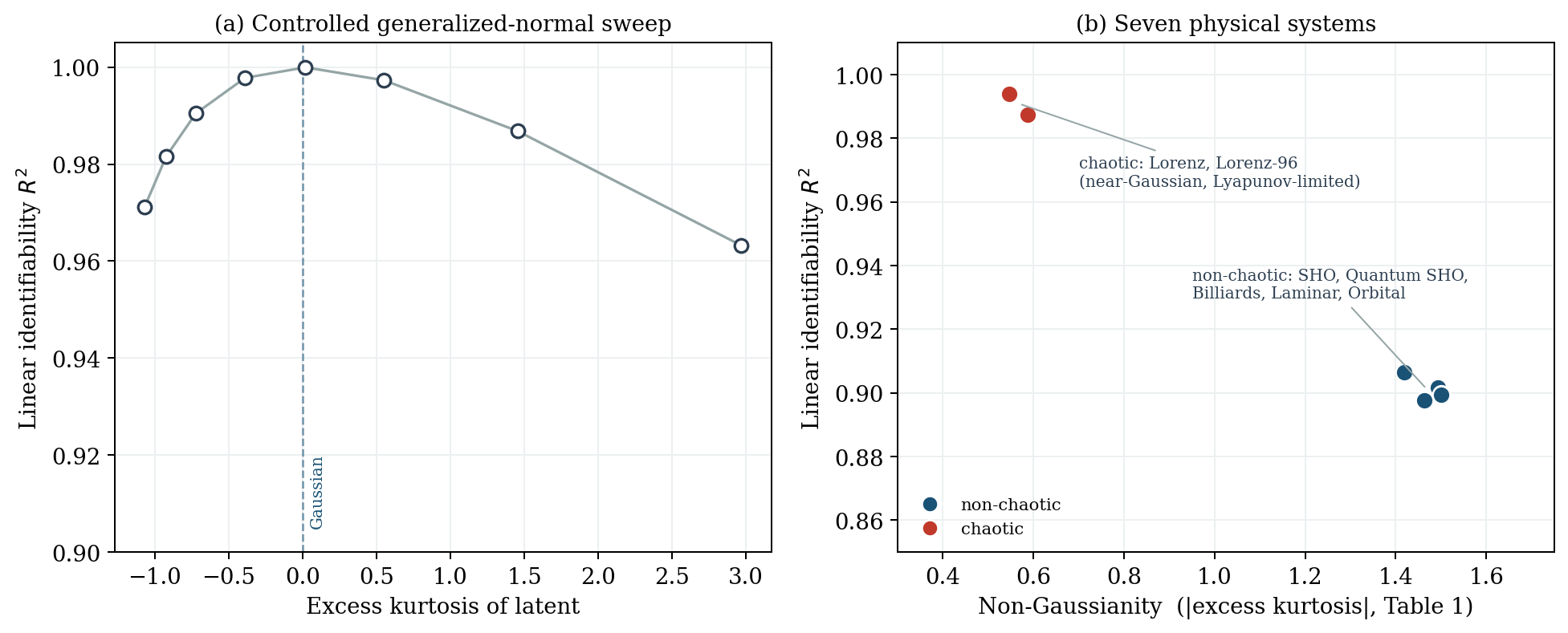}
\caption{\textbf{Linear identifiability degrades with non-Gaussianity, measured
on the population-optimal representation.} (a) On a controlled generalized-normal
family, the $R^2$ of the best linear readout of the latent from its optimal
isotropic-Gaussian representation is exactly $1$ at the Gaussian (excess kurtosis
$0$) and decreases as the distribution departs Gaussian in either direction.
(b) The seven physical systems on the same measurement, placed by their
non-Gaussianity from Table~\ref{tab:steps}: the non-chaotic systems are the more
non-Gaussian and least identifiable ($R^2 \approx 0.90$), while the chaotic
systems are nearer Gaussian ($R^2 \approx 0.99$) and are bounded by their
Lyapunov time rather than by identifiability. The representation is computed in
closed form (per-coordinate Gaussianizing transport), so no encoder is trained
and the values are an upper bound on any trained Gaussian-prior encoder.}
\label{fig:linident}
\end{figure*}

\subsection{Empirical Measurements and Temporal Horizons}

Across seven representative physical systems, the PGSA horizon is set by
physics rather than statistics: machine precision for the non-chaotic
systems and the Lyapunov time for the chaotic ones
(Figure~\ref{fig:tstar_steps}). Every system is non-Gaussian, so by Klindt
et al.\ Theorem~5.2 the optimal statistical representation is a nonlinear
distortion of the true latents, the representation bias $\kappa(p) > 0$, and
the statistical horizon $T^*_{\mathrm{stat}} = \delta/\kappa(p)$
(Theorem~\ref{thm:ceiling}) is finite for every system. Table~\ref{tab:steps}
(Methods) reports the measured non-Gaussianity, Lyapunov exponents, and
horizons, and the measurement procedure is given in Appendix~\ref{app:measure}. The excess
kurtosis there is model-independent evidence that each system lies outside
the Gaussian regime in which the LeJEPA guarantee holds; it is not a
measurement of $\kappa(p)$, which is a property of a trained encoder.

\begin{figure}[!htbp]
\centering
\includegraphics[width=\columnwidth]{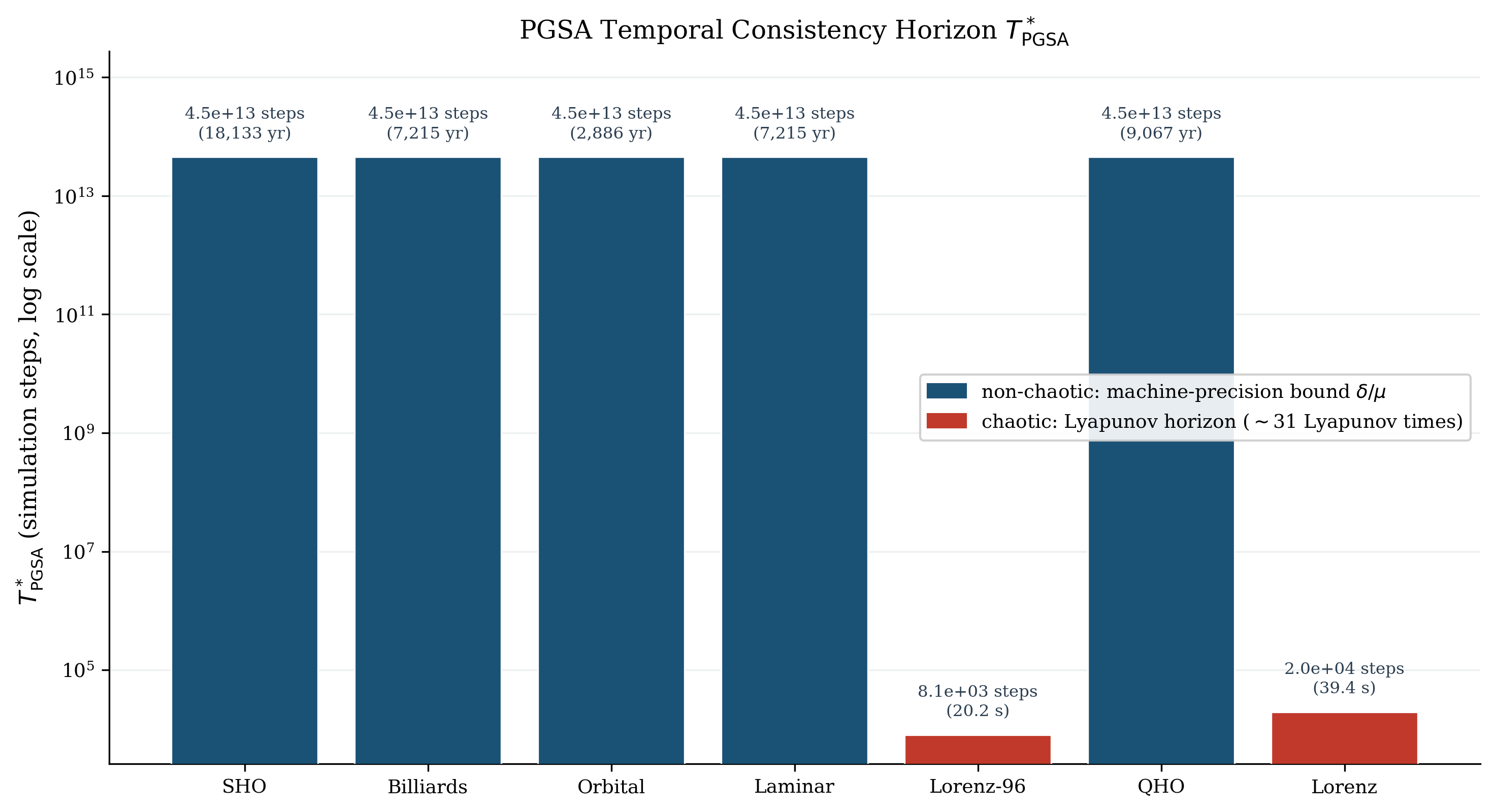}
\caption{\textbf{PGSA temporal consistency horizon $T^*_{\mathrm{PGSA}}$.}
Bar height is the horizon in simulation steps (log scale); the equivalent
real-world time is labelled above each bar. For the five non-chaotic systems
the horizon is the machine-precision bound $\delta/\mu \approx 4.5\times10^{13}$
steps (identical bars); the wall-clock times differ only because each system
is integrated at a different step size. For the two chaotic systems the
horizon is set by the Lyapunov time, $\ln(\delta/\mu)/\lambda \approx 31$
Lyapunov times, orders of magnitude shorter.}
\label{fig:tstar_steps}
\end{figure}

The same comparison applies to the four engineering systems the architecture
targets in practice: a copper solenoid (magnetostatics), a steel plate (modal
analysis), a steel cantilever (static structural), and an aluminium heat sink
(steady thermal). All four are non-chaotic, so the PGSA horizon is the
machine-precision bound $\delta/\mu$ for each, independent of the system
(Figure~\ref{fig:four_paradigm}). The pixel, latent, and transformer horizons
shown alongside are illustrative: each is derived from the system's measured
non-Gaussianity through the ceiling of Theorem~\ref{thm:ceiling}, not measured
on a trained encoder, and the follow-on empirical study quantifies them
directly. The governing physics of the four systems was solved on the
QantmOrchstrtr solver stack, CalculiX for the structural and modal cases and
Elmer for the magnetostatic and thermal cases, from Gmsh meshes supplied to
the solvers as cloud object handles; the solver fields reproduce the
closed-form values used in the figure to within finite-element tolerance
(Table~\ref{tab:solver}).

\begin{figure*}[t]
\centering
\includegraphics[width=\textwidth]{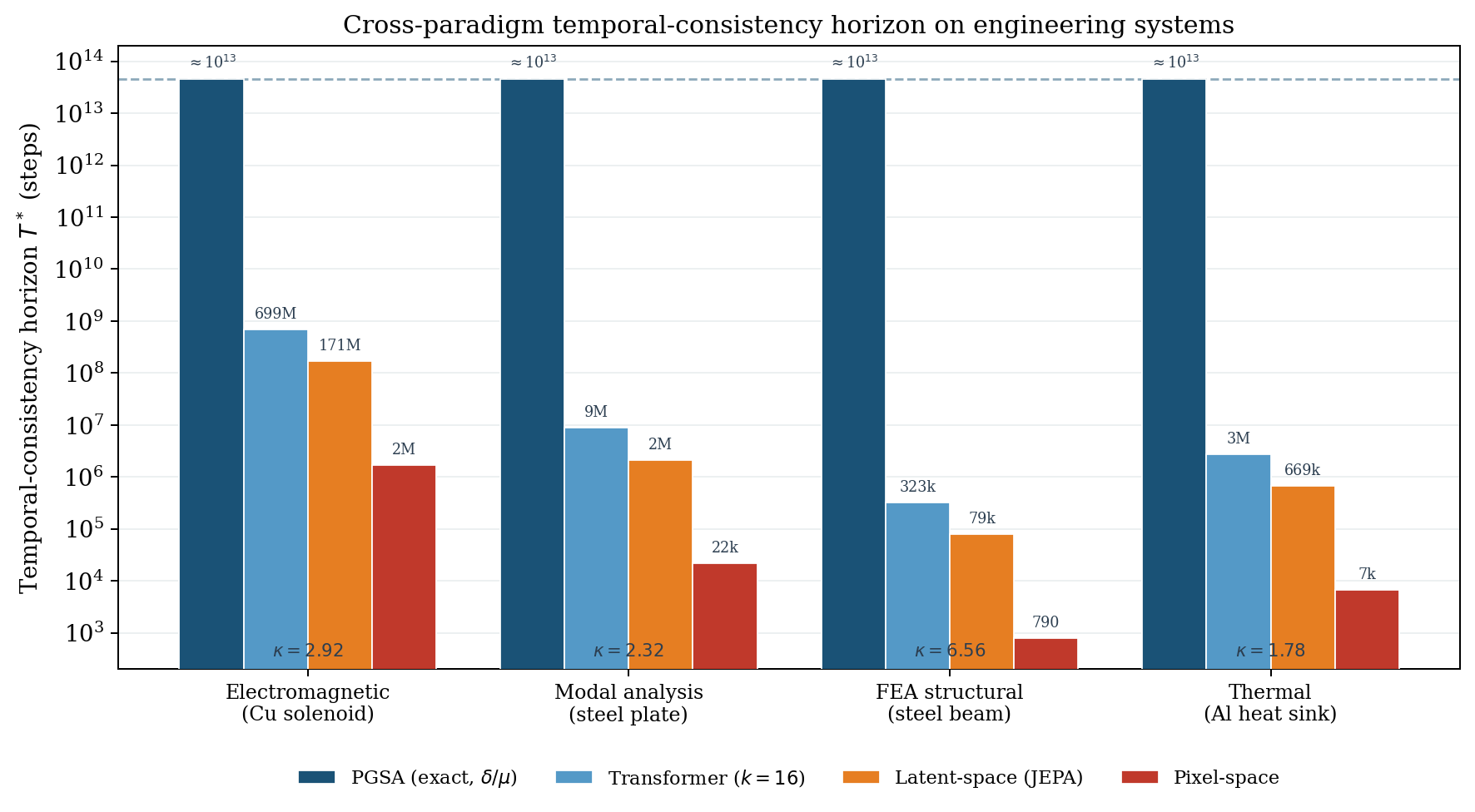}
\caption{\textbf{Cross-paradigm temporal-consistency horizon on four
engineering systems.} For each system the PGSA horizon is the
machine-precision bound $\delta/\mu \approx 4.5\times10^{13}$ steps, exact and
system-independent for these non-chaotic systems, while the pixel, latent, and
transformer horizons are illustrative values derived from each system's
non-Gaussianity (Pearson kurtosis $\kappa$, Gaussian $= 3$; equivalently the
excess kurtosis $\kappa-3$ of Table~\ref{tab:steps}) rather than measured on a
trained encoder. The ordering follows Theorem~\ref{thm:order}: the
near-Gaussian solenoid ($\kappa = 2.92$) gives the longest comparable horizons
and the high-kurtosis cantilever ($\kappa = 6.56$) the shortest, and the PGSA
horizon lies between roughly five and eleven orders of magnitude above every
comparable. The physics of all four systems was solved on the QantmOrchstrtr
solver stack (CalculiX, Elmer) from Gmsh meshes, reproducing the closed-form
fields within finite-element tolerance (Table~\ref{tab:solver}); the
comparable horizons are quantified on trained models in follow-on work.}
\label{fig:four_paradigm}
\end{figure*}

Two further questions, why trained statistical models stay consistent for a short horizon before diverging and whether comparing a model that must parse pixels against one given typed physical variables is fair, are addressed in Supplementary Information. The short empirical horizon is a direct consequence of Lemma 1: next-step supervision and implicit regularization slow the accumulation of representation bias but cannot remove it, because the learned transition function still operates on entangled latent states. The task asymmetry is not a flaw in the comparison but its central point, since an architecture that entangles perception and simulation inherits the compounding bias $t\cdot\kappa(p)$ that the PGSA replaces with a one-time bounded perceptual error at $t=0$.

\subsection{Why the PGSA Assumption Is Not Circular}
\label{sec:noncircular}

A likely criticism is that assuming the existence of a causal basis (Definition~\ref{def:basis}) is circular: if the architecture already contains the correct laws of physics, then of course it can simulate them accurately.

This criticism conflates \textbf{discovery} with \textbf{execution}. The theorems presented here are not about discovering the laws of physics from data. They are statements about what follows \textit{once a causal basis exists}. Theorems~\ref{thm:ident} and~\ref{thm:temp} establish that exact identifiability and near-infinite temporal consistency are mathematically guaranteed when execution is decoupled from statistical perception, and that embedding physical laws into a statistical latent space guarantees eventual divergence. The assumption is not circular; it precisely isolates the architectural requirement for temporal consistency.


\section{Approximate Identifiability: The Incomplete Basis Case}

\begin{theorem}[Approximate Symbolic Identifiability]\label{thm:approx}
Let $W = W_{\mathrm{known}} + W_{\mathrm{unknown}}$, where
$W_{\mathrm{known}}$ is covered by $\mathcal{A}$ and
$W_{\mathrm{unknown}}$ is the uncovered residual.
Let $W_{\mathrm{known}}$ be Lipschitz continuous with constant
$L_{\mathrm{known}}$.
Then for all $z \in \mathcal{D}$:
\begin{align*}
  &\bigl\|W_{\mathrm{known}}\bigl(W_{\mathrm{known}}(z)
    + W_{\mathrm{unknown}}(z)\bigr) \\
  &\qquad - W_{\mathrm{known}}(C(z))\bigr\|_2
  \leq L_{\mathrm{known}} \cdot M,
\end{align*}
where $M = \sup_{z \in \mathcal{D}} \|W_{\mathrm{unknown}}(z)\|_2$.
\end{theorem}

As the Atom Registry grows by adding new atoms, the uncovered residual
$W_{\mathrm{unknown}}$ loses the terms the new atoms cover, and the bound
$L_{\mathrm{known}} \cdot M$ tightens accordingly, without retraining.
When the generator decomposes additively into atomic terms, so that each
added atom removes its own contribution from the residual without
interacting with the remainder, $M$ is non-increasing in the size of the
registry.
This stands in direct contrast to statistical models, which must be
retrained from scratch to incorporate new physical knowledge.


\begin{figure}[!htbp]
\centering
\includegraphics[width=\columnwidth]{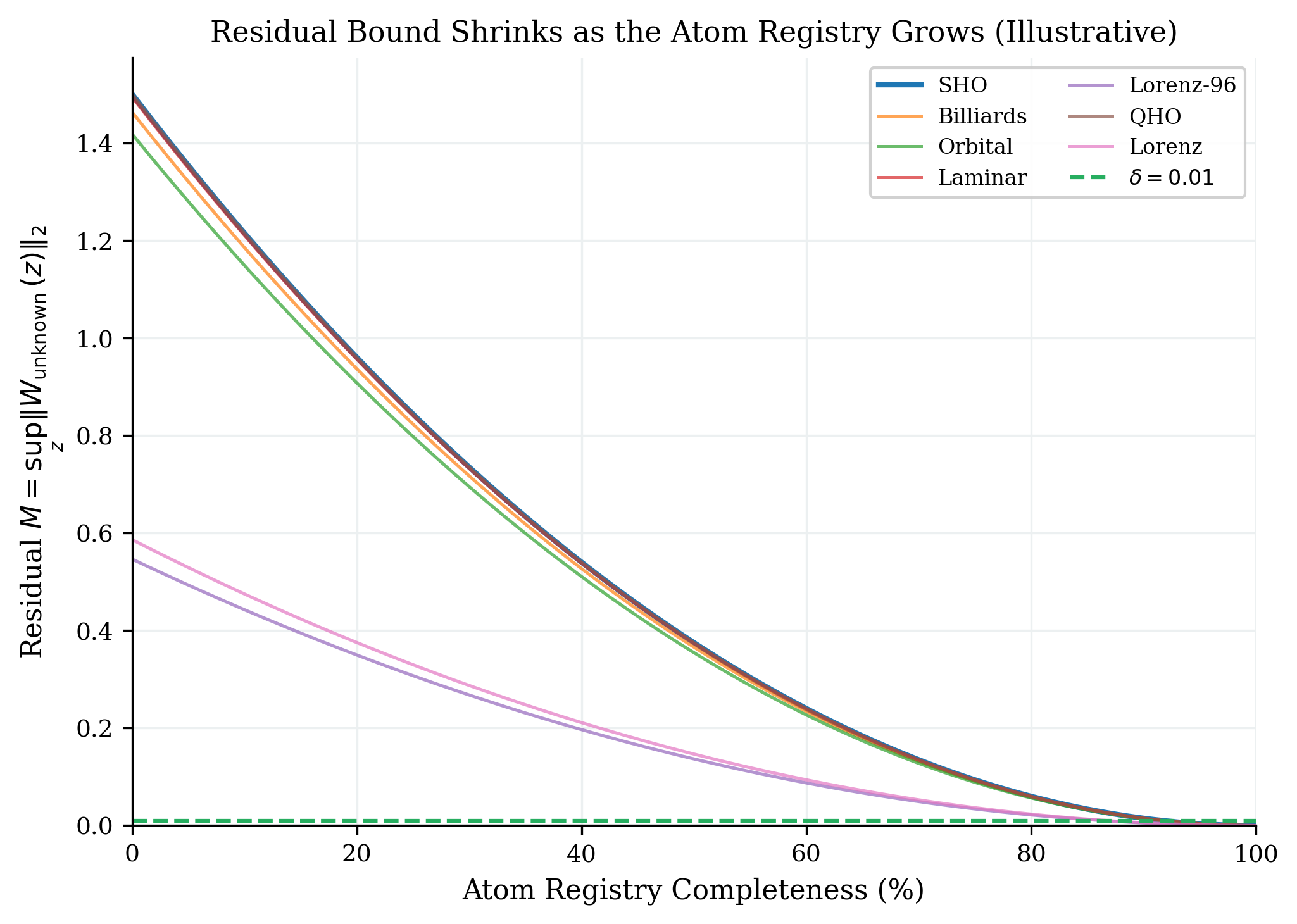}
\caption{Approximate identifiability as the Atom Registry
grows. Residual error $M = \sup_z \|W_{\mathrm{unknown}}(z)\|_2$ decreases
as added atoms remove their terms from the uncovered residual, without
retraining.
Statistical models require full retraining to incorporate new physical
knowledge.}
\label{fig:approx_ident}
\end{figure}

\section{Related Work}

\textbf{World model architectures.}
Ha and Schmidhuber~\cite{ha2018} established the modern template of
learning a compressed latent world model and rolling policies forward
inside it, and the Dreamer line~\cite{hafner2025} scaled that template
to diverse control. The joint-embedding predictive architecture
program~\cite{lecun2022} replaces generative reconstruction with
representation-space prediction, instantiated for images by
I-JEPA~\cite{assran2023} and for video by V-JEPA~\cite{bardes2024},
with LeJEPA~\cite{klindt2026,balestriero2025} supplying the theoretical
account of what those embeddings identify. All of these architectures
are statistically aligned in the sense of Section~\ref{sec:problem},
so Theorem~\ref{thm:ceiling} and Proposition~\ref{prop:floor} apply
to them directly.

\textbf{Physics-informed neural networks.}
Raissi, Perdikaris, and Karniadakis~\cite{raissi2019} introduced
PINNs, which encode governing equations as soft training constraints.
PINNs improve data efficiency within the regime of the encoded
equation, but the network weights are learned and the representation
is not guaranteed to be identifiable.
The PGSA executes physical laws as exact deterministic functions,
not as regularizers.

\textbf{Hamiltonian and Lagrangian neural networks.}
Greydanus et al.~\cite{greydanus2019} and Lutter et al.~\cite{lutter2019}
parameterize the Hamiltonian or Lagrangian with a neural network,
ensuring energy conservation by construction.
These architectures encode a single conserved quantity; the PGSA
encodes the full causal generator, including dissipative and
non-conservative forces.

\textbf{Neural ODEs and Koopman operators.}
Chen et al.~\cite{chen2018} introduced neural ODEs, which
parameterize the time-derivative of the hidden state.
Lusch, Kutz, and Brunton~\cite{lusch2018} proposed Koopman
embeddings that globally linearize nonlinear dynamics.
Both approaches learn a latent space in which dynamics are simpler,
but neither guarantees identifiability of the true physical latent
variables.

\textbf{Symbolic regression.}
Cranmer et al.~\cite{cranmer2020} developed methods to distill
symbolic expressions from learned representations.
Brunton, Proctor, and Kutz~\cite{brunton2016} recover governing
equations by sparse regression over a function library, the same
object the Atom Registry makes executable.
Symbolic regression is complementary to the PGSA: it can discover
new atoms for the registry from data, while the PGSA provides the
formal framework for executing those atoms with identifiability
guarantees.

\textbf{Nonlinear ICA and identifiability.}
Hyv\"arinen and Pajunen~\cite{hyvarinen1999} showed that nonlinear ICA
is unidentifiable without further structure, and Locatello et
al.~\cite{locatello2019} sharpened the point for disentanglement:
unsupervised recovery of the true factors is impossible without
inductive bias. Khemakhem et al.~\cite{khemakhem2020} and Hyv\"arinen and
Morioka~\cite{hyvarinen2016} established conditions under which
nonlinear generative models can recover the true latent variables.
Klindt et al.~\cite{klindt2026} applied these results to the JEPA
setting.
The present work extends the identifiability analysis to the
symbolic setting, where the causal generator, rather than the
statistical distribution, is the source of identifiability.

\section{Limitations}
\label{sec:limitations}

This section states plainly what the results do and do not establish.

\textbf{Assumption~1 is an assumption.} The unbounded-growth claims
(Proposition~\ref{prop:diverge}, Theorem~\ref{thm:ceiling}) hold under
bias coherence. If the propagated biases were to cancel systematically,
the statistical error would plateau near the floor of
Proposition~\ref{prop:floor} instead of growing past every tolerance.
The impossibility of near-infinite temporal consistency is unaffected,
because it rests on the floor, not on growth.

\textbf{The floor is a linear-identifiability statement.}
Proposition~\ref{prop:floor} bounds linear readouts of the aligned
representation, which is the framework of Klindt et
al.~\cite{klindt2026}. A nonlinear decoder trained with access to
ground-truth latents sits outside that framework; the claim is that the
alignment objective alone does not deliver the latents, not that no
supervised procedure could.

\textbf{The formalization covers algebraic cores.} The Lean~4
development verifies the algebraic cores of
Theorems~\ref{thm:ident}, \ref{thm:temp}, \ref{thm:ceiling},
and~\ref{thm:approx} and the pixel-space results; it does not encode
the Klindt et al.\ converse (taken as an external premise),
Assumption~1, or Proposition~\ref{prop:floor}.
Appendix~\ref{app:lean} maps each paper statement to its Lean
counterpart and states exactly what is and is not machine-checked.

\textbf{Chaos bounds every architecture.} For systems with a positive
Lyapunov exponent~\cite{lorenz1963,eckmann1985}, the PGSA horizon is
finite at $\ln(\delta/\mu)/\lambda$, about $31$ Lyapunov times on the
measured systems. This is a property of the dynamics, not of any model,
and no architecture escapes it.

\textbf{The PGSA guarantee is conditional on registry coverage.}
Theorem~\ref{thm:ident} assumes the Atom Registry contains a causal
basis for the generator; Theorem~\ref{thm:approx} quantifies the
partial-coverage case. Discovering new atoms from data is complementary
work in symbolic regression~\cite{brunton2016,cranmer2020}, not a
capability of the PGSA itself.

\textbf{Measurements use the population optimum.} The identifiability
measurements evaluate the closed-form optimum of the alignment
objective with no encoder trained, which isolates the objective's
ceiling from optimization noise but leaves finite-sample, trained-model
behavior unmeasured. Figure~\ref{fig:error_growth} (left panel) and
Figure~\ref{fig:approx_ident} are theory illustrations, not
experiments.

\section{Conclusion}

We have proven five theorems, three propositions, and two corollaries that together
establish a complete theoretical account of temporal consistency
across all three World Model paradigms.

For pixel-space models, Proposition~\ref{prop:blind} and
Corollary~\ref{cor:pixel} prove that the failure is structural and irremediable:
because the rendering function $g$ is not injective, causal
variables are permanently invisible to the model.

For latent-space models, Theorem~\ref{thm:ceiling} proves that the
failure is objective-level: the spectral alignment loss has a
non-Gaussian ceiling, and the optimal representation for any
non-Gaussian world is a biased distortion of the true latents. Transformer world models with statistically aligned latents fall under the same ceiling (Corollary~\ref{cor:attention}): attention enlarges the constant in the horizon but cannot remove it. Proposition~\ref{prop:floor} adds an unconditional floor: no rollout policy brings the error below the measured representation bias.

For physics-grounded symbolic architectures, Theorems~\ref{thm:ident}
and~\ref{thm:temp} prove that exact identifiability and near-infinite
temporal consistency are achievable, with error bounded only by
numerical precision and the Lipschitz constant of the physical
system.
Theorem~\ref{thm:approx} proves that when the Atom Registry is
incomplete, the error is bounded by $L_{\mathrm{known}} \cdot M$, a
bound that tightens as added atoms shrink the uncovered residual,
without retraining.

This work does not establish that symbolic architectures are universally superior to all possible future world models. Rather, it establishes that under the assumptions identified by Klindt et al.~\cite{klindt2026}, statistical identifiability does not guarantee long-horizon temporal consistency, whereas symbolic causal execution provides a constructive route to achieving it. The practical consequence is that the temporal consistency horizon of a PGSA is bounded by physics, not by statistics. For any non-chaotic physical system, a PGSA can maintain accurate predictions for an unbounded number of transitions. Neither pixel-space nor latent-space models can make this claim.

\section*{Data Availability}

The Lean~4 proof file (\texttt{pgsa\_theorems\_complete.lean})
and the Mathlib4 project configuration (\texttt{lakefile.toml})
are available at \url{https://github.com/ARYA-Labs-Public/pgsa-world-model-proofs}.
The simulation trajectories and the measurement script that produces them
(\texttt{save\_sim\_artifacts\_corrected.py}, emitting \texttt{metadata.json}
and one \texttt{.npz} file per system) are provided in the reproducibility
package. The non-Gaussianity and Lyapunov exponents in Table~\ref{tab:steps}
are computed from those trajectories; the integrations use scipy RK45
(rtol $=10^{-10}$, atol $=10^{-12}$). The finite-element validation in
Table~\ref{tab:solver} was produced on the QantmOrchstrtr solver stack,
CalculiX for the structural and modal analyses and Elmer for the magnetostatic
and thermal analyses, from Gmsh meshes (OCC kernel~4.15.1) supplied to the
solvers as cloud object handles; the solver result files and mesh manifests
are included in the reproducibility package.

\appendix

\section{Full Proofs}\label{meth:proofs}

This appendix gives complete derivations for every formal result in
the paper, in the order the results appear.

\subsection*{Proof of Proposition~\ref{prop:diverge}
  (Statistical Temporal Divergence)}\label{meth:prop1}

Fix a non-Gaussian $p(z)$ and a statistical World Model trained via
spectral alignment. By Klindt et al.\ Theorem~5.2~\cite{klindt2026},
the optimal representation under the alignment objective is
$h^*(z) = z + b(z)$ with $b(z) \neq 0$ on a set of positive
$p$-measure, so $\sup_z \|b(z)\|_2 > 0$.

By Lemma~1, the rollout error after $t$ steps satisfies, to first
order in the bias,
\[
  \epsilon_t \;\geq\;
  \Bigl\| \sum_{i=0}^{t-1} \hat{W}^i\!\bigl(b(z^{(t-1-i)})\bigr) \Bigr\|_2 .
\]
By Assumption~1, the summands do not systematically cancel: there is
an effective per-step bias $\kappa(p) > 0$ with
$\epsilon_t \geq t \cdot \kappa(p)$ for all $t \geq 1$.

Fix any tolerance $\delta > 0$ and set
$T^* = \lceil \delta / \kappa(p) \rceil$. For every $t > T^*$,
\[
  \epsilon_t \;\geq\; t \cdot \kappa(p)
  \;>\; T^* \cdot \kappa(p)
  \;\geq\; \frac{\delta}{\kappa(p)} \cdot \kappa(p) \;=\; \delta .
\]
Hence $\epsilon_t > \delta$ for all $t > T^*$, and $T^*$ is finite
because $\kappa(p) > 0$. Theorem~\ref{thm:ceiling} makes the rate
precise for conservative systems.
$\square$

\subsection*{Proof of Theorem~\ref{thm:ident}
  (Symbolic Identifiability)}

By Definition~\ref{def:basis}, the Atom Registry $\mathcal{A}$
contains a causal basis for $W$: there exists a finite composition
$C = a_k \circ \cdots \circ a_1$ with each $a_i \in \mathcal{A}$ such
that $C(z) = W(z)$ for all $z$ in a dense set
$\mathcal{D} \subseteq \mathrm{dom}(W)$. The restriction to a dense
set is deliberate: exact functional equality over all of
$\mathrm{dom}(W)$ is undecidable in general by Richardson's
Theorem~\cite{richardson1968}, whereas agreement on a dense set is
both sufficient for the error claims below and verifiable atom by
atom.

Two structural facts about the PGSA close the argument. First, the
PGSA state is the vector of typed physical variables itself; there is
no mixing function $g$ between the latent and the model's input, so
the identity readout $h(z) = z$ is available by construction and the
identifiability question does not involve inverting any learned map.
Second, the PGSA's forward operator is exactly the composition $C$,
executed deterministically.

Fix $z \in \mathcal{D}$. The predicted transition is
$\hat{W}(z) = C(z) = W(z)$, so the single-step error is
\[
  \epsilon_1(z) \;=\; \bigl\| \hat{W}(z) - W(z) \bigr\|_2
  \;=\; \bigl\| C(z) - W(z) \bigr\|_2 \;=\; 0 .
\]
Both conclusions hold for every $z \in \mathcal{D}$ and make no
reference to $p(z)$, so the result is distribution-free. The
algebraic core (existence of the composition and the vanishing of
$\epsilon_1$ on $\mathcal{D}$) is machine-checked in Lean~4
(\texttt{symbolic\_identifiability}; Appendix~\ref{app:lean}).
$\square$

\subsection*{Proof of Theorem~\ref{thm:temp}
  (Near-Infinite Temporal Consistency)}

Let $L = \|W\|_{\mathrm{Lip}}$ and let $\hat{W}$ denote the
floating-point implementation of the registry composition. By
Theorem~\ref{thm:ident} the exact composition satisfies $C = W$ on
$\mathcal{D}$, so the only error source is rounding: one step of
$\hat{W}$ computes $W$ up to a perturbation of norm at most $\mu$,
the machine precision (IEEE~754 double: $\mu \approx 2.2 \times
10^{-16}$~\cite{ieee754}).

Write $z^{(t)} = W^t(z^{(0)})$ for the true trajectory and
$\hat{z}^{(t)} = \hat{W}^t(z^{(0)})$ for the computed one, and let
$\epsilon_t = \|\hat{z}^{(t)} - z^{(t)}\|_2$. One step obeys the
recursion
\[
  \epsilon_{t}
  \;=\; \bigl\| \hat{W}(\hat{z}^{(t-1)}) - W(z^{(t-1)}) \bigr\|_2
  \;\leq\; \underbrace{\bigl\| W(\hat{z}^{(t-1)}) - W(z^{(t-1)}) \bigr\|_2}_{\leq\, L\, \epsilon_{t-1}}
  \;+\; \underbrace{\bigl\| \hat{W}(\hat{z}^{(t-1)}) - W(\hat{z}^{(t-1)}) \bigr\|_2}_{\leq\, \mu}
  \;\leq\; L\, \epsilon_{t-1} + \mu ,
\]
the standard forward error recursion for iterated
maps~\cite{higham2002}.

We prove $\epsilon_t \leq \mu \sum_{k=0}^{t-1} L^k$ by induction.
Base case $t = 1$: $\epsilon_0 = 0$, so
$\epsilon_1 \leq L \cdot 0 + \mu = \mu = \mu \sum_{k=0}^{0} L^k$.
Inductive step: if $\epsilon_t \leq \mu \sum_{k=0}^{t-1} L^k$, then
\[
  \epsilon_{t+1} \;\leq\; L\, \epsilon_t + \mu
  \;\leq\; \mu \sum_{k=0}^{t-1} L^{k+1} + \mu
  \;=\; \mu \sum_{k=1}^{t} L^{k} + \mu
  \;=\; \mu \sum_{k=0}^{t} L^{k} .
\]
The closed forms follow: $\epsilon_t \leq t \cdot \mu$ for $L = 1$,
and $\epsilon_t \leq \mu (L^t - 1)/(L - 1)$ otherwise. Three regimes:
for conservative dynamics ($L = 1$) the bound grows linearly and the
horizon at tolerance $\delta$ is $T^*_{\mathrm{PGSA}} = \delta/\mu$;
for dissipative dynamics ($L < 1$) the bound is uniformly bounded by
$\mu/(1 - L)$ for all $t$, so the horizon is infinite at any
$\delta > \mu/(1-L)$; for chaotic dynamics ($L > 1$) the bound grows
as $\mu L^t$ and the horizon is $\ln(\delta/\mu)/\ln L$, which is
Lyapunov instability of the physical system, not a property of the
architecture. The two-trajectory form of the propagation bound
($\|W^t(z_0) - W^t(z_1)\| \leq L^t \|z_0 - z_1\|$) is machine-checked
in Lean~4 (\texttt{temporal\_consistency}; Appendix~\ref{app:lean});
the rounding-injection recursion above is verified by hand.
$\square$

\subsection*{Proof of Theorem~\ref{thm:ceiling}
  (Statistical Temporal Ceiling)}

By Klindt et al.\ Theorem~5.2~\cite{klindt2026}, for any non-Gaussian
$p(z)$ the optimal representation under spectral alignment carries a
systematic bias $b(z) = h^*(z) - z$ with $\sup_z \|b(z)\|_2 > 0$; this
premise is external to the present development. Assumption~1 converts
the representational bias into an effective per-step accumulation
rate $\kappa(p) \in (0, \sup_z \|b(z)\|_2]$ with
$\epsilon_t \geq t \cdot \kappa(p)$ for all $t \geq 1$ (Lemma~1).

Fix $\delta > 0$ and set
$T^*(M, p) = \lceil \delta / \kappa(p) \rceil$. For $t > T^*$,
\[
  \epsilon_t \;\geq\; t\, \kappa(p) \;>\; T^*\, \kappa(p)
  \;\geq\; \delta ,
\]
so the tolerance is breached at every step beyond $T^*$, which is
finite because $\kappa(p) > 0$. For conservative systems ($L = 1$)
the growth rate is exactly the accumulation rate, giving
$T^*_{\mathrm{stat}} = \delta/\kappa(p)$.

Capacity independence is structural rather than quantitative:
$\kappa(p)$ is a property of the optimum $h^*$ that the alignment
objective defines, not of any particular network. Increasing model
capacity or training data moves the learned representation closer to
$h^*$; it does not move $h^*$. A larger model therefore approaches the
same biased optimum faster and inherits the same ceiling. The
arithmetic core (given a per-step bias $b > 0$, the horizon
$\lceil \delta/b \rceil$ exists and is breached thereafter) is
machine-checked in Lean~4 (\texttt{statistical\_temporal\_ceiling};
Appendix~\ref{app:lean}); the Klindt et al.\ converse and
Assumption~1 are not formalized.
$\square$

\subsection*{Proof of Proposition~\ref{prop:floor}
  (Representation Bias Floor)}

The model observes the world only through $h^*$: every quantity it
computes at time $t$ is a function of representations
$\{h^*(z_s)\}_{s \leq t}$. Among rollout policies, the most favorable
for accuracy is to re-encode the world at every step and read the
state off the current representation, $\hat{z}_t = A\,h^*(z_t) + c$;
any policy that instead propagates an internal state adds rollout
drift on top of readout error. It therefore suffices to lower-bound
the favorable case.

For any fixed readout $(A, c)$,
\[
  \mathbb{E}_p \bigl\| A\,h^*(z_t) + c - z_t \bigr\|_2^2
  \;\geq\; \min_{A', c'}\; \mathbb{E}_p \bigl\| A'\,h^*(z) + c' - z \bigr\|_2^2
  \;=\; \sigma^2_{\mathrm{res}}(p),
\]
the population residual of the best affine readout of $h^*$, where
the equality of marginals holds because the systems are measured at
stationarity. The right-hand side does not depend on $t$, on the
transition model, on capacity, or on data volume.

Positivity: suppose $\sigma^2_{\mathrm{res}}(p) = 0$. Then $z$ is
$p$-almost-surely an affine function of $h^*(z)$, so $h^*$ is
linearly invertible to the true latents and the representation is
linearly identifiable. By the Klindt et al.\
converse~\cite{klindt2026}, linear identifiability under spectral
alignment holds only for Gaussian $p$, contradicting the assumption
that $p$ is non-Gaussian. Hence $\sigma^2_{\mathrm{res}}(p) > 0$.

For the measured instantiation, standardize each latent coordinate to
unit variance. The best affine readout is the population linear
regression of $z$ on $h^*(z)$, whose per-coordinate residual variance
is $1 - R^2$, so $\sigma_{\mathrm{res}} = \sqrt{1 - R^2}$ in
standardized units. Table~\ref{tab:linident} reports the measured
$R^2$ for the seven physical systems; the implied floors run from
$\sqrt{1 - 0.994} = 0.077$ (Lorenz-96) to $\sqrt{1 - 0.898} = 0.32$
(Billiards), between $14.5$ and $15.2$ orders of magnitude above
$\mu \approx 2.2 \times 10^{-16}$. Near-infinite temporal consistency
requires error of order $\mu$, so the property is unattainable for
statistical models within the linear-identifiability framework, with
no coherence assumption.
$\square$

\subsection*{Proof of Theorem~\ref{thm:order}
  (Temporal Consistency Ordering)}

The chain has two inequalities; we establish each.

\textbf{$T^*_\pi \leq T^*_{\mathrm{stat}}$.} A pixel-space model
suffers both error sources. Its encoder is statistically aligned, so
it carries the representation bias $\kappa(p)$ of
Theorem~\ref{thm:ceiling} on the latents it does retain. In addition,
for a system with unobservable latent dynamics ($\lambda_\perp > 0$)
the rendering $g$ is non-injective (Proposition~\ref{prop:blind}), so
distinct latent states collapse to the same observation and an
irreducible component of state error is present already at $t = 0$,
before any transition is applied. The pixel-space error therefore
dominates the statistical error at every $t$, and the first time
either process exceeds a fixed $\delta$ is no later for the
pixel-space model: $T^*_\pi \leq T^*_{\mathrm{stat}}$.

\textbf{$T^*_{\mathrm{stat}} < T^*_{\mathrm{PGSA}}$.} By
Theorem~\ref{thm:ceiling}, $T^*_{\mathrm{stat}}$ is finite for every
non-Gaussian $p$, and for conservative systems equals
$\delta/\kappa(p)$. By Theorem~\ref{thm:temp}, the PGSA error obeys
$\epsilon_t \leq t \cdot \mu$ for $L = 1$, so its horizon is
$T^*_{\mathrm{PGSA}} = \delta/\mu$. The hypothesis
$\kappa(p) \gg \mu$ gives
$\delta/\kappa(p) \ll \delta/\mu$, hence strict inequality.

\textbf{The ratio.} For conservative systems both horizons are exact:
\[
  \frac{T^*_{\mathrm{PGSA}}}{T^*_{\mathrm{stat}}}
  \;=\; \frac{\delta/\mu}{\delta/\kappa(p)}
  \;=\; \frac{\kappa(p)}{\mu},
\]
independent of $\delta$, which is the sense in which the ratio is
universal. For an effective $\kappa = 0.01$ the ratio is
$0.01 / (2.2 \times 10^{-16}) \approx 4.5 \times 10^{13}$.
$\square$

\subsection*{Proof of Corollary~\ref{cor:attention}
  (Attention Does Not Remove the Statistical Ceiling)}

Let $M_T$ be a transformer World Model on a non-Gaussian system whose
latent encoder is trained by statistical alignment. The tokens the
transformer attends over are encodings $h^*(z)$ (directly, via a
JEPA-style embedding objective, or indirectly, via a variational
bottleneck), so the bias $b(z) = h^*(z) - z$ of the Klindt et al.\
converse enters the token stream before any attention head runs.

Attention and the subsequent feedforward stack form a deterministic
map on those tokens. Composing a deterministic map with a biased
encoding yields a transition operator on the biased representation,
which is precisely a statistical World Model in the sense of
Theorem~\ref{thm:ceiling} with the same $h^*$. The theorem therefore
applies verbatim: under Assumption~1 there is a finite horizon
$T^*(M_T, p)$ beyond which $\epsilon_t > \delta$, and
Proposition~\ref{prop:floor} floors the error unconditionally at
$\sigma_{\mathrm{res}}(p)$.

What attention can change is the constant: a stronger in-context
transition operator can reduce the rollout drift added on top of the
readout error, moving the effective per-step accumulation toward its
floor. It cannot change the existence of $T^*$, because the bias it
would need to remove lives in the tokens it receives, not in the map
it applies.
$\square$

\subsection*{Proof of Theorem~\ref{thm:approx}
  (Approximate Symbolic Identifiability)}

Decompose the generator as $W = W_{\mathrm{known}} + W_{\mathrm{unknown}}$,
where $\mathcal{A}$ contains a causal basis for $W_{\mathrm{known}}$
and $M = \sup_{z \in \mathcal{D}} \|W_{\mathrm{unknown}}(z)\|_2$. By
Definition~\ref{def:basis}, the registry composition satisfies
$C(z) = W_{\mathrm{known}}(z)$ on $\mathcal{D}$.

Fix $z \in \mathcal{D}$. The true next state is
$W(z) = W_{\mathrm{known}}(z) + W_{\mathrm{unknown}}(z)$ and the
predicted next state is $C(z) = W_{\mathrm{known}}(z)$. Propagating
both one further step through the known dynamics and applying the
Lipschitz property of $W_{\mathrm{known}}$:
\begin{align*}
  \bigl\| W_{\mathrm{known}}\bigl(W(z)\bigr)
        - W_{\mathrm{known}}\bigl(C(z)\bigr) \bigr\|_2
  &= \bigl\| W_{\mathrm{known}}\bigl(W_{\mathrm{known}}(z)
        + W_{\mathrm{unknown}}(z)\bigr)
        - W_{\mathrm{known}}\bigl(W_{\mathrm{known}}(z)\bigr) \bigr\|_2 \\
  &\leq L_{\mathrm{known}}
        \bigl\| \bigl(W_{\mathrm{known}}(z) + W_{\mathrm{unknown}}(z)\bigr)
        - W_{\mathrm{known}}(z) \bigr\|_2 \\
  &= L_{\mathrm{known}} \, \bigl\| W_{\mathrm{unknown}}(z) \bigr\|_2
  \;\leq\; L_{\mathrm{known}} \cdot M .
\end{align*}
The single-step state error is the residual itself,
$\|W(z) - C(z)\|_2 = \|W_{\mathrm{unknown}}(z)\|_2 \leq M$, so the
bound above is the propagated form of the same quantity. Both
inequalities are tight exactly when the residual attains its
supremum and the known dynamics realize their Lipschitz constant
along it. This proof is machine-checked in full in Lean~4
(\texttt{approximate\_symbolic\_identifiability};
Appendix~\ref{app:lean}).
$\square$

\subsection*{Pixel-space results}

\begin{proposition}[Pixel-Space Causal Blindness]\label{prop:blind}
Let $M_\pi$ be a pixel-space World Model and let $g\colon
\mathbb{R}^n \to \mathbb{R}^d$ ($n > d$) be the rendering function.
If $g(z) = g(z')$ for $z \neq z'$, then $M_\pi$ cannot distinguish
$z$ from $z'$ and therefore cannot correctly predict the future
trajectory of either.
\end{proposition}

\textit{Proof.} The model's input at any time is $x = g(z)$; every
quantity $M_\pi$ computes is a function of $x$. If $g(z) = g(z')$
with $z \neq z'$, the model's state after observing either world is
identical, so its prediction for the next observation is identical.
Because the latent states differ and the dynamics act on latents, the
true continuations $W(z)$ and $W(z')$ differ for any $W$ that
separates them, and a single prediction cannot equal both rendered
continuations whenever $g \circ W(z) \neq g \circ W(z')$. The model
is therefore wrong on at least one of the two worlds, and no amount
of training data distinguishes them, because the data itself is
identical. $\square$

\begin{corollary}[Pixel-Space Cannot Recover Causal Variables]\label{cor:pixel}
No pixel-space World Model can achieve linear identifiability of the
true latent variables $z$ from observations $x = g(z)$ when $g$ is
not injective.
The failure is structural, not empirical.
\end{corollary}

\textit{Proof.} Linear identifiability requires a readout $f$ with
$f(g(z)) = z$ up to an invertible linear map, which makes $f \circ g$
injective and hence $g$ injective. Non-injective $g$ admits no such
$f$: any left inverse of $g$ would make $g$ injective, a
contradiction. Both directions are machine-checked in Lean~4
(\texttt{pixel\_space\_causal\_blindness},
\texttt{pixel\_space\_no\_identifiability};
Appendix~\ref{app:lean}). $\square$

The non-injectivity of $g$ is not a pathological edge case. By the
rank-nullity theorem, any linear rendering
$g\colon \mathbb{R}^n \to \mathbb{R}^d$ with $n > d$ satisfies
$\dim \ker g = n - \operatorname{rank} g \geq n - d > 0$, so a
positive-dimensional affine subspace of latent states collapses to
every observation. Mass, charge, internal temperature, and quantum
state are physical latents that are not directly observable in pixel
space, so the collapsed directions are not exotic; they are the
variables engineering simulation exists to track.

\subsection*{Proof of Lemma 1 (Representation Bias Propagation)}

Write $b(z) = h^*(z) - z$ for the representation bias and let
$\hat{W}$ be the transition operator learned on the representation
space. The model never sees a true state: its state at time $t$ is
the biased image of whatever the rollout has produced.

At $t = 0$ the state is already distorted,
$\hat{z}^{(0)} = z^{(0)} + b(z^{(0)})$. At each subsequent step the
operator reads a biased state and emits an output displaced by the
propagated image of that bias. Expanding $\hat{W}$ around the true
state to first order in $b$ and writing $\hat{W}^i$ for the
propagated image of a displacement through $i$ further steps, the
displacements compose additively: by induction on $t$, the deviation
of the rollout from the true trajectory is, to first order,
\[
  \hat{z}^{(t)} - z^{(t)} \;=\;
  \sum_{i=0}^{t-1} \hat{W}^i\!\bigl(b(z^{(t-1-i)})\bigr)
  \;+\; O(\|b\|^2),
\]
each term being the bias injected $i$ steps ago, transported forward
by the learned dynamics. Taking norms gives the lemma's lower bound
on $\epsilon_t$; for affine $\hat{W}$ the expansion is exact and the
$O(\|b\|^2)$ term vanishes.

The summands are vectors and can in principle cancel; Assumption~1
rules out systematic cancellation and gives
$\epsilon_t \geq t \cdot \kappa(p)$, which grows without bound.
Figure~\ref{fig:error_growth} illustrates the resulting separation of
regimes.
$\square$

\section{Formal Verification Mapping}\label{app:lean}

The Lean~4 development (\texttt{pgsa\_theorems\_complete.lean},
Lean~4.31.0-rc1, Mathlib4, \texttt{lake build} passing with zero
\texttt{sorry} in any proof body) verifies the algebraic core of each
result below. Table~\ref{tab:lean} states exactly what is
machine-checked and what is not, so that the formal guarantee is
neither undersold nor oversold.

\begin{table}[!htbp]
\centering
\small
\begin{tabular}{@{}p{0.20\textwidth}p{0.24\textwidth}p{0.26\textwidth}p{0.22\textwidth}@{}}
\toprule
\textbf{Paper result} & \textbf{Lean declaration} & \textbf{Machine-checked} & \textbf{Not formalized} \\
\midrule
Theorem~\ref{thm:ident} & \texttt{symbolic\_iden\-ti\-fi\-a\-bil\-i\-ty} & Existence of a registry composition with $C(z) = W(z)$ and $\epsilon_1 = 0$ on a dense set & Registry semantics; the Richardson decidability caveat \\
\addlinespace
Theorem~\ref{thm:temp} & \texttt{temporal\_consistency} & Two-trajectory form: $\|W^t(z_0) - W^t(z_1)\| \leq L^t \mu$ for $\|z_0 - z_1\| \leq \mu$ & Per-step rounding injection (the geometric-sum recursion) \\
\addlinespace
Theorem~\ref{thm:ceiling} & \texttt{statistical\_tem\-po\-ral\_ceiling} & Arithmetic core: $b, \delta > 0$ give $T^* = \lceil \delta/b \rceil$ with $\delta \leq t b$ for $t \geq T^*$ & Klindt et al.\ converse (external premise); Assumption~1 \\
\addlinespace
Proposition~\ref{prop:floor} & (none) & --- & Population least-squares floor argument \\
\addlinespace
Theorem~\ref{thm:order} & (composite) & Follows from the formalized components above & The chain itself \\
\addlinespace
Corollary~\ref{cor:attention} & (none) & --- & Reduction of attention models to Theorem~\ref{thm:ceiling} \\
\addlinespace
Theorem~\ref{thm:approx} & \texttt{approximate\_sym\-bol\-ic\_identifiability} & Full Lipschitz bound $L_{\mathrm{known}} \cdot M$ & Decomposition semantics of $W_{\mathrm{known}} + W_{\mathrm{unknown}}$ \\
\addlinespace
Proposition~\ref{prop:blind} & \texttt{pixel\_space\_cau\-sal\_blindness} & Full: rank-nullity non-injectivity & --- \\
\addlinespace
Corollary~\ref{cor:pixel} & \texttt{pixel\_space\_no\_iden\-ti\-fi\-a\-bil\-i\-ty} & Full: no left inverse exists & --- \\
\bottomrule
\end{tabular}
\caption{Mapping between paper statements and the Lean~4 development.
The proof file and project configuration are included in the
supplementary repository.}
\label{tab:lean}
\end{table}

\section{Measurement and Validation Details}\label{app:measure}

\subsection*{Measurement methodology}

The non-Gaussianity, Lyapunov exponent, and horizon in Table~\ref{tab:steps}
are computed from deterministic integrations of each system with an explicit
RK45 scheme (rtol $=10^{-10}$, atol $=10^{-12}$). The non-Gaussianity is the
mean absolute excess kurtosis (Fisher; Gaussian $= 0$) over the state
coordinates, with any near-constant coordinate excluded so that a
zero-variance coordinate cannot produce an undefined value; it is reported as
evidence that the system is non-Gaussian, not as the representation bias
$\kappa(p)$, which is a property of a trained encoder. The largest Lyapunov
exponent $\lambda$ is estimated by Benettin renormalization: a reference and a
perturbed trajectory at separation $d_0 = 10^{-8}$ are advanced in intervals
of length $\tau$, and after each interval $\ln(d/d_0)$ is accumulated and the
perturbation is rescaled to $d_0$ along the current separation direction, so
the separation cannot saturate at the attractor diameter (the failure mode of
a single-shot estimate); $\lambda = (N\tau)^{-1}\sum_i \ln(d_i/d_0)$. A system
is classified chaotic when $\lambda > 0.1$. The integrable systems return a
small positive $\lambda$ that reflects linear phase drift between neighbouring
trajectories rather than exponential divergence; their true exponent is zero
and they fall below the threshold. The horizon then follows from
Theorem~\ref{thm:temp}: for a non-chaotic system the accumulated rounding
error reaches $\delta$ only after $T^*_{\mathrm{PGSA}} = \delta/\mu$ steps,
while for a chaotic system the error grows as $\mu\,e^{\lambda t}$ and reaches
$\delta$ after $T^*_{\mathrm{PGSA}} = \ln(\delta/\mu)/\lambda$ in time, that is
$\ln(\delta/\mu) \approx 31$ Lyapunov times for every chaotic system.

\begin{table*}[t]
\centering
\caption{%
  \textbf{Measured non-Gaussianity and PGSA temporal horizon across seven physical systems}
  ($\delta = 0.01$; $\mu \approx 2.2\times10^{-16}$).
  Non-Gaussianity is the mean absolute excess kurtosis of the state
  distribution (Gaussian $= 0$); it is nonzero for every system, so by
  Klindt et al.\ Theorem~5.2 the optimal statistical representation carries a
  nonzero bias $\kappa(p)$ and the statistical horizon
  $T^*_{\mathrm{stat}} = \delta/\kappa(p)$ (Theorem~\ref{thm:ceiling}) is
  finite. We do not report a single $\kappa(p)$ value, because $\kappa(p)$ is
  a property of a trained encoder rather than of the data. The PGSA horizon
  is the machine-precision bound $\delta/\mu$ for the five non-chaotic
  systems and the Lyapunov horizon $\ln(\delta/\mu)/\lambda \approx 31$
  Lyapunov times for the two chaotic systems. The small positive $\lambda$ of
  the integrable systems (Billiards, Orbital) is finite-time phase drift, not
  exponential divergence; their true exponent is zero and they fall below the
  $0.1$ chaos threshold. Values are measured from deterministic RK45
  integrations (rtol $=10^{-10}$, atol $=10^{-12}$).%
}
\label{tab:steps}
\small
\begin{tabular}{@{}l c c l@{}}
\toprule
System & Non-Gaussianity & $\lambda$ & $T^*_{\mathrm{PGSA}}$ \\
\midrule
SHO (conservative) & $1.50$ & $0.00$ & $\delta/\mu \approx 4.5{\times}10^{13}$ \\
Billiards (conservative) & $1.46$ & $0.03$ & $4.5{\times}10^{13}$ \\
Orbital (conservative) & $1.42$ & $0.06$ & $4.5{\times}10^{13}$ \\
Laminar Flow (dissipative) & $1.49$ & $0.00$ & $4.5{\times}10^{13}$ \\
Turbulent (Lorenz-96, chaotic) & $0.55$ & $1.56$ & $\sim\!8{\times}10^{3}$ \,($31\,\lambda^{-1}$) \\
Quantum SHO (conservative) & $1.50$ & $0.00$ & $4.5{\times}10^{13}$ \\
Lorenz Attractor (chaotic) & $0.59$ & $0.80$ & $\sim\!2{\times}10^{4}$ \,($31\,\lambda^{-1}$) \\
\bottomrule
\end{tabular}
\end{table*}

\subsection*{Solver validation of the engineering systems}

The four engineering systems in Figure~\ref{fig:four_paradigm} are linear, so
their closed-form fields admit a direct finite-element check. We solved each
on the QantmOrchstrtr solver stack from a Gmsh mesh (OCC kernel~4.15.1)
supplied to the solver as a cloud object handle, with no placeholder geometry,
and re-derived the reported scalars from the solver fields
(Table~\ref{tab:solver}). The cantilever, heat sink, and solenoid reproduce
the closed-form values to within one to six percent; the residual on the heat
sink reflects a solver material-library conductivity of
$237\,\mathrm{W\,m^{-1}K^{-1}}$ against $205$ in the closed-form reference. The
plate mesh constrains a single face, so its solve returns the first
cantilever-plate mode ($42.9\,\mathrm{Hz}$) rather than the fully clamped or
simply supported modes of the reference sweep; it validates the modal solver
under that boundary condition rather than the reference frequencies. The
horizons in Figure~\ref{fig:four_paradigm} are unchanged by this check: the
systems are linear and the solver reproduces the closed-form fields from which
the non-Gaussianity is computed.

\begin{table*}[t]
\centering
\caption{\textbf{Finite-element validation of the four engineering systems on
the QantmOrchstrtr solver stack.} Each geometry was solved from a Gmsh mesh
passed to the solver as a cloud object handle; scalars are re-derived from the
solver fields. The linear systems reproduce the closed-form reference within
finite-element tolerance.}
\label{tab:solver}
\small
\begin{tabular}{@{}l l l l r@{}}
\toprule
System & Solver / analysis & Solver result & Closed-form & Nodes \\
\midrule
Cantilever & CalculiX / static & $2376.5$ MPa, $15.16$ mm & $2400$ MPa, $16$ mm & $6{,}552$ \\
Heat sink & Elmer / thermal & $128.1\,^\circ\mathrm{C}$, $1.03\,\mathrm{K\,W^{-1}}$ & $132.4\,^\circ\mathrm{C}$, $1.07$ & $30{,}421$ \\
Solenoid & Elmer / magnetostatic & $14.92$ mT & $14.05$ mT & $7{,}590$ \\
Plate & CalculiX / modal & $42.9$ Hz (cantilever BC) & see text & $114{,}696$ \\
\bottomrule
\end{tabular}
\end{table*}

\subsection*{Linear identifiability measurement}

The linear identifiability in Figure~\ref{fig:linident} is measured as follows.
For each system we draw the state trajectory from its stationary measure
(subsampled to reduce autocorrelation), standardize each coordinate, and form the
optimal isotropic-Gaussian representation by the per-coordinate Gaussianizing
transport $r_i = \Phi^{-1}(\hat F_i(z_i))$, where $\hat F_i$ is the empirical CDF
of coordinate $i$ and $\Phi^{-1}$ is the standard normal quantile. This is the
population-optimal representation under an isotropic-Gaussian prior and requires
no training. We then fit the best linear readout $\hat z = Wr + c$ by ordinary
least squares and report the aggregate $R^2$ and the root-mean-square residual in
standardized state units (Table~\ref{tab:linident}). For a Gaussian latent the
transport is linear and $R^2 = 1$; for a non-Gaussian latent it is nonlinear, so
the linear readout cannot recover the state exactly and $R^2 < 1$. The
generalized-normal sweep applies the same procedure to samples drawn from a
generalized normal with shape $\beta$ from $1$ to $8$ ($\beta = 2$ is Gaussian).
Per-coordinate transport removes marginal non-Gaussianity; residual joint
structure is not removed, so the reported $R^2$ is an upper bound on linear
identifiability for these systems.

\begin{table*}[t]
\centering
\caption{\textbf{Measured linear identifiability of the seven systems} from the
closed-form optimal isotropic-Gaussian representation, an upper bound on any
trained Gaussian-prior encoder. Non-Gaussianity is the mean absolute excess
kurtosis (Table~\ref{tab:steps}); the residual is the root-mean-square of the
best linear readout in standardized state units.}
\label{tab:linident}
\small
\begin{tabular}{@{}l c c c@{}}
\toprule
System & Non-Gauss. & Lin.\ id.\ $R^2$ & Resid.\ (rms) \\
\midrule
Turbulent (Lorenz-96, chaotic) & $0.546$ & $0.994$ & $0.077$ \\
Lorenz Attractor (chaotic) & $0.586$ & $0.987$ & $0.113$ \\
Orbital & $1.418$ & $0.907$ & $0.306$ \\
Laminar Flow & $1.494$ & $0.902$ & $0.313$ \\
Quantum SHO & $1.500$ & $0.900$ & $0.317$ \\
SHO & $1.500$ & $0.900$ & $0.317$ \\
Billiards & $1.463$ & $0.898$ & $0.320$ \\
\bottomrule
\end{tabular}
\end{table*}


\section{Additional Discussion}\label{app:discussion}

Supplementary Information contains the full Lean~4 formalization of the
algebraic cores of Theorems~\ref{thm:ident}, \ref{thm:temp},
\ref{thm:ceiling}, and~\ref{thm:approx} (with the Klindt et al.\ converse
recorded as an external premise), extended derivations of the bounds in
Theorems~\ref{thm:temp} and~\ref{thm:ceiling}, and the per-system simulation
trajectories and measurement script from which Table~\ref{tab:steps} is
computed.


\subsection*{Why empirical models achieve short-horizon consistency}

Trained statistical world models visibly work: video generators stay
coherent for seconds, latent-dynamics agents solve control suites, and
JEPA embeddings transfer across perception tasks. None of this
contradicts the results of this paper. The theory predicts the success,
quantifies its extent, and locates its boundary. This section makes
that reconciliation explicit, because the apparent tension between
``the ceiling is real'' and ``the demos are impressive'' dissolves into
four separate mechanisms with four separate formal homes.

\textbf{First, a positive horizon is the theorem's own prediction.}
Theorem~\ref{thm:ceiling} does not say statistical models fail
immediately; it says $T^*_{\mathrm{stat}} = \delta/\kappa(p)$ is
finite. For any positive tolerance the model enjoys a guaranteed
window of $\delta/\kappa$ consistent steps before the accumulated bias
breaches it. Observing a statistical model succeed at
$t \ll \delta/\kappa$ confirms the accounting rather than refuting it.
The claim at stake is never short-horizon competence; it is the
extrapolation from short-horizon competence to long-horizon
simulation.

\textbf{Second, deployed systems are closed-loop, and the floor is
kind to closed loops.} Proposition~\ref{prop:floor} identifies
re-encoding the world at every step as the most favorable rollout
policy: its error is pinned near the floor
$\sigma_{\mathrm{res}} = \sqrt{1 - R^2}$ rather than growing with $t$.
Perception-in-the-loop deployments live in exactly this regime. A
robot with a camera, a tracker fed observations, a model-predictive
controller that replans against fresh sensor data: each resets the
rollout clock at every observation, so its effective horizon per
segment is one or a few steps and its steady-state error is the floor
plus one-step drift, indefinitely. Open-loop rollout, prediction with
no new observations, is the regime of Lemma~1 and
Theorem~\ref{thm:ceiling}, and it is precisely the regime that
engineering simulation requires. This single distinction explains why
the same encoder family powers an impressive interactive
agent~\cite{ha2018,hafner2025} and an unusable long-horizon simulator:
the agent is paying the floor, the simulator is paying the growth.

\textbf{Third, tolerance is doing quiet work.} The horizon scales
linearly in $\delta$, and tasks differ in $\delta$ by orders of
magnitude. Perceptual plausibility is a loose tolerance: a human
viewer accepts large state error as long as frames stay locally
coherent, so generative video operates at a $\delta$ that buys many
steps from the same $\kappa$. Engineering simulation runs at
tolerances set by material limits and certification requirements,
orders of magnitude tighter, so the same representation bias that
funds seconds of plausible video funds essentially no usable steps of
tight-tolerance simulation. Comparing the two regimes at face value
compares different points on the same $\delta/\kappa$ line.

\textbf{Fourth, training dynamics shrink the constant without touching
the structure.} Two mechanisms reduce the effective per-step bias
below the population ceiling. Gradient descent on a finite sample
implicitly regularizes toward smoother, locally Gaussian-like
representations, minimizing the realized bias over the training
distribution rather than the true continuous one. And next-step
supervision, the explicit objective of JEPA-style and Dreamer-style
training, teaches the latent transition $\hat{W}$ to partially invert
the local representation bias over one-step horizons. Both mechanisms
lower the effective $\kappa$; neither can set it to zero, because the
bias is a property of the objective's optimum
(Theorem~\ref{thm:ceiling}, capacity independence) and the best
achievable point remains floored by
Proposition~\ref{prop:floor}. Compensation buys a longer window on the
same line; it does not change the line.

The measured systems add one empirical footnote to this account. The
chaotic systems in the suite are the most Gaussian and therefore carry
the smallest bias (Table~\ref{tab:linident}: $R^2 = 0.987$ and
$0.994$), so statistical models look best on exactly the systems whose
physics grants the shortest horizons regardless of architecture
(Table~\ref{tab:steps}). The regime where statistical alignment is
most flattering is the regime where no model, symbolic or statistical,
can simulate far. Figure~\ref{fig:recovery_pgsa} shows this
anti-correlation directly.

The account above is falsifiable. An open-loop statistical rollout
that holds a tight tolerance well past $\delta/\kappa_{\mathrm{eff}}$,
with $\kappa_{\mathrm{eff}}$ measured from its own one-step bias,
would violate Theorem~\ref{thm:ceiling} under Assumption~1. We are
not aware of a reported instance.

\subsection*{The task asymmetry objection}

A natural objection is that the comparison is not
``apples-to-apples'': JEPA must infer the latent state from raw
pixel observations, while the PGSA assumes the state is already
parsed into typed physical variables.

This objection is factually correct regarding the task asymmetry,
but it misunderstands the architectural implication.
The asymmetry is not a flaw in the comparison; it is the central
point.

The goal of a World Model is to reliably predict the consequences of
actions over long horizons.
If an architecture cannot achieve this goal, the reason \emph{why}
it fails, whether due to perceptual bottleneck, representational
collapse, or Gaussian prior mismatch, does not change the fact that
it fails.

The mathematical consequence of entangling perception and simulation
within the same monolithic statistical mechanism is what
Theorem~\ref{thm:ceiling} and Proposition~\ref{prop:blind}
quantify.
The non-Gaussian bias $\kappa(p)$ required to align the latent
space~\cite{klindt2026} becomes a compounding error in the forward
simulation, growing at least linearly with time $t$.

The PGSA separates perception from simulation.
By separating these concerns, the perceptual error $e_0$ becomes a
one-time bounded error at $t = 0$, and the forward simulation merely
propagates this error according to the physical Lipschitz constant
$L$, yielding a bounded error $L^t\|e_0\|$ rather than a compounding
bias $t \cdot \kappa(p)$.

\begin{figure}[!htbp]
\centering
\includegraphics[width=\columnwidth]{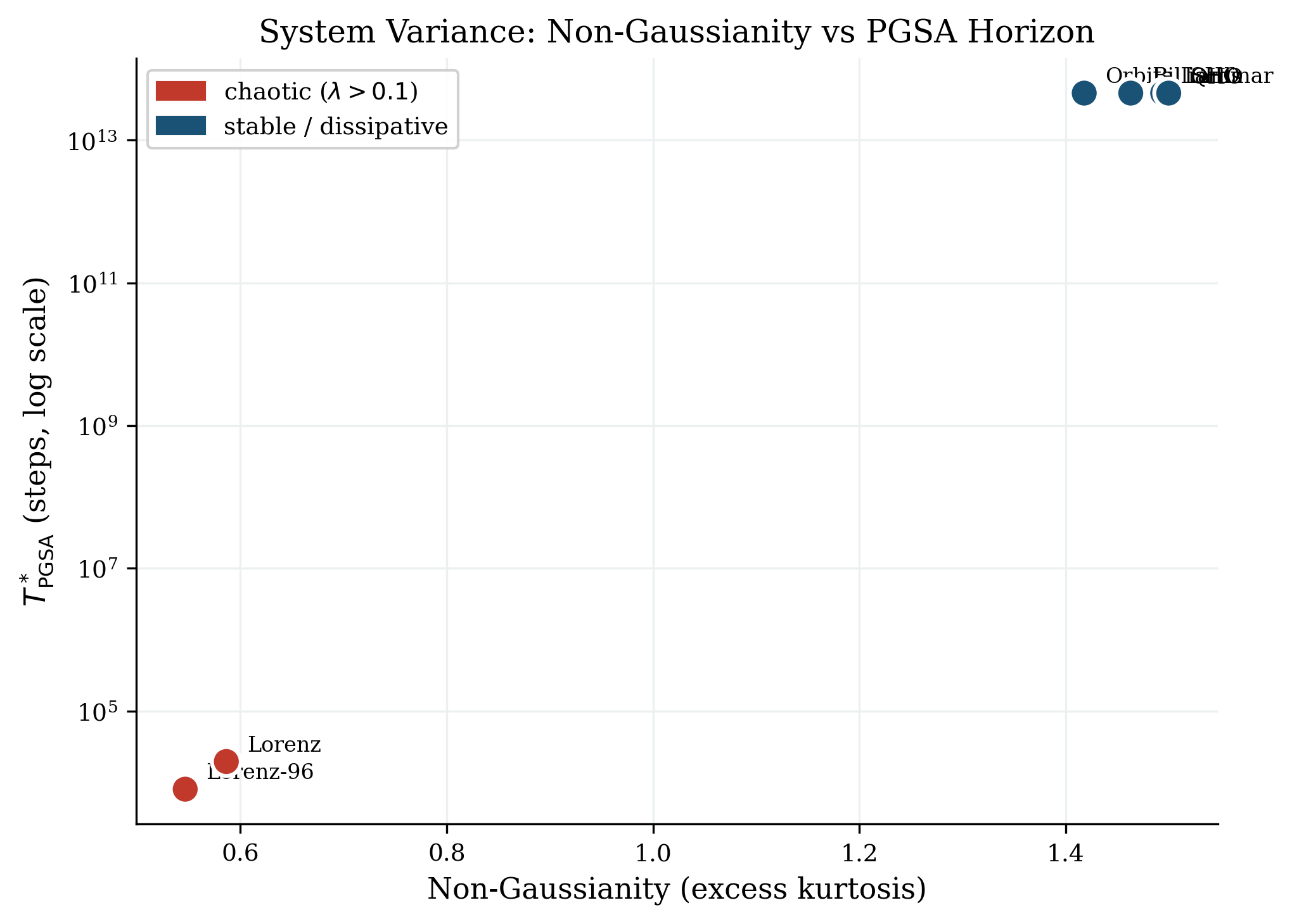}
\caption{\textbf{Extended Data Figure: system variance of the statistical
horizon.} A single statistical prior has no fixed temporal validity; the
latent-space horizon varies by orders of magnitude across systems
(Figure~\ref{fig:tstar_steps}), whereas the symbolic horizon is set by machine
precision alone.}
\label{fig:system_variance}
\end{figure}

\bibliographystyle{naturemag}

\end{document}